\documentclass[11pt]{article}
\usepackage[margin=1in]{geometry}
\usepackage{amsmath,amssymb,amsthm,mathtools}
\usepackage{booktabs}
\usepackage{enumitem}
\usepackage{algorithm}
\usepackage{algpseudocode}
\usepackage[colorlinks=true,linkcolor=blue,citecolor=blue,urlcolor=blue]{hyperref}

\newtheorem{theorem}{Theorem}[section]
\newtheorem{lemma}[theorem]{Lemma}
\newtheorem{proposition}[theorem]{Proposition}
\newtheorem{corollary}[theorem]{Corollary}

\theoremstyle{remark}
\newtheorem{remark}[theorem]{Remark}
\newtheorem{example}[theorem]{Example}

\newcommand{\OPT}{\mathrm{OPT}}
\newcommand{\ALG}{\mathrm{ALG}}
\newcommand{\LE}{\mathrm{LE}}
\newcommand{\E}{\mathbb{E}}
\renewcommand{\Pr}{\mathbb{P}}
\newcommand{\R}{\mathbb{R}}
\newcommand{\PC}{\mathrm{PC}}
\newcommand{\RMP}{\mathrm{RMP}}
\newcommand{\SRMP}{\mathrm{SRMP}}
\newcommand{\NPA}{\mathrm{NPA}}
\newcommand{\RET}{\mathrm{RET}}
\newcommand{\APC}{\mathrm{APC}}
\newcommand{\SRMPc}{\mathrm{SRMP}^{\mathrm{cf}}}
\newcommand{\PhiK}{\Phi}
\DeclareMathOperator*{\argmax}{arg\,max}

\title{Prior-Free Competitive Ratios for Improving Bandits:\\ Scale, Curvature and Horizon Are Free, but Not Jointly Under Noise}
\newif\ifanonymous\anonymousfalse
\ifanonymous
  \author{}
\else
  \author{Xuan Li\\[2pt]
  \small University of New South Wales, Sydney, Australia \quad \texttt{winny.li@unsw.edu.au}\\[2pt]
  \small \href{https://orcid.org/0009-0002-0213-6991}{ORCID: 0009-0002-0213-6991}}
\fi
\date{}

\begin{document}
\maketitle

\begin{abstract}
In the improving multi-armed bandits problem, each of $k$ arms has an unknown nondecreasing, discretely concave reward curve $f_i$, and pulling arm $i$ for the $t$-th time yields $f_i(t)$. For sufficiently long horizons, Blum and Ravichandran (ALT 2025) proved that randomized algorithms achieve an $O(\sqrt k)$ approximation to the best single arm when the scale $m=f^*(T)$ of the optimal arm is known ($T\ge2k$), and $O(\sqrt k\log k)$ when it is not ($T>4k$), against an $\Omega(\sqrt k)$ lower bound. The logarithmic factor is unnecessary: a one-page \emph{probe-and-commit} algorithm achieves competitive ratio $4\sqrt3\,\sqrt k$ for $T\ge2\lfloor\sqrt k\rfloor$, without any knowledge of the scale, and we determine the optimal ratio for every horizon, $\Theta(\sqrt k+k/T)$, also for unknown horizons. Without noise, \emph{no prior is needed at all}: a random-marginal probing algorithm reading neither the scale $m$, nor the concavity-envelope exponent $\beta$ of Blum, Garicano, Ravichandran and Sharma (UAI 2026), nor the horizon $T$, achieves the optimal $\Theta(k^{\beta/(1+\beta)}+k/T)$ simultaneously for every $\beta$ and every horizon. Under the multiplicative noise model of Blum and Ravichandran, probe-and-commit keeps the same all-horizon order $\Theta(\sqrt k+k/T)$ without knowing the noise level (and $\Theta(\sqrt k)$ on the same range), but the price of priors jumps: for any fixed noise level $\varepsilon\in(0,1/2]$, the uniform price of adaptation $\phi_\varepsilon(k)$ --- the worst case over horizons $T\ge16k$ of the loss relative to $k^{\beta/(1+\beta)}$ for algorithms knowing neither $m$ nor $\beta$ --- is $\Theta_\varepsilon(\sqrt{\log k/\log\log k})$, the lower bound asymptotic in $k$ at fixed positive $\varepsilon$ and matched by a nested random-permutation probing algorithm, whereas knowing either $m$ or $\beta$ alone restores a constant price.
\end{abstract}

\section{Introduction}

The \emph{improving multi-armed bandits} problem (IMAB), formalized by Heidari, Kearns and Roth~\cite{HKR16} and Patil et al.~\cite{Pat23}, models the allocation of effort among $k$ options whose payoff grows with the effort invested in them: pulling arm $i$ for the $t$-th time yields $f_i(t)$, where each $f_i$ is nondecreasing with diminishing returns (discretely concave) and $f_i(0)=0$. The benchmark is the best single arm, $\OPT_T=\max_i\sum_{t\le T}f_i(t)$, and since sublinear regret is impossible one measures the \emph{competitive ratio} $\OPT_T/\E[\ALG_T]$.

Throughout, \emph{curvature} in the title refers to the lower-envelope exponent $\beta$ of \cite{BGRS26}---how fast a curve is required to rise relative to its terminal value---and not to strong concavity in any second-derivative sense; all our curves are merely discretely concave. Patil et al.~\cite{Pat23} study deterministic horizon-unaware improving bandits and state an $O(k)$ competitive guarantee. We use their deterministic lower-bound construction as background --- the lower bound we restate in Section~\ref{sec:discussion} uses only determinism, so it covers horizon-aware deterministic algorithms as well --- but we establish our range-qualified deterministic bounds independently there, and we do not invoke the literal all-horizon upper bound of the algorithm displayed in arXiv:2208.09254v1: its two-pulls-per-arm initialisation already incurs ratio $k(2k+1)/3$ at $T=2k$ on a single linear arm with zero decoys. We use this characterisation only on the sufficient range $T\ge2k$, and assert no sharp boundary: Section~\ref{sec:discussion} shows that for $T\le k$ the deterministic ratio is $k(k+1)/2$ or unbounded, while at $T=2k-1$ it is still $\Theta(k)$. The horizon qualification is ours to state and is not attributed to~\cite{Pat23}. Blum and Ravichandran~\cite{BR25} proved that randomization helps: no algorithm beats $\Omega(\sqrt k)$; a random round-robin with an absolute threshold $m\,t/T$ achieves $O(\sqrt k)$ \emph{if it is told} the value $m=f^*(T)$ of the optimal arm in advance; and, in their words, one can ``remove this assumption at the cost of an extra $O(\log k)$ approximation factor, achieving an overall $O(\sqrt k\log k)$ approximation relative to optimal'' (for $T>4k$). Blum, Garicano, Ravichandran and Sharma~\cite{BGRS26} refined the picture with a concavity-envelope exponent $\beta\in(0,1]$: if every arm satisfies $f_i(t)\ge f_i(T)(t/T)^\beta$, then a thresholded round robin achieves the optimal $\Theta(k^{\beta/(1+\beta)})$, again assuming $m$ is known, and they handle unknown $\beta$ by learning a parameter from offline instances; in the context of their data-driven hybrid they write: ``We note that as a result of this approach, we do not get per-instance worst-case guarantees. This is an interesting direction for future work.''\footnote{\cite{BGRS26}, introduction of Appendix~D (discussion of the data-driven regret/competitive-ratio hybrid and of when to switch); see also Appendix~D.2 there, in the same context: ``Thus, that algorithm may not provide the worst-case fallback guarantee on a fixed instance. As mentioned earlier, this is an interesting consideration for future work.'' Both remarks concern their hybrid algorithm; they do not name the unknown-scale question. Our per-instance prior-free guarantees address one instance of the direction they describe.}

For the general concave class, \cite{BR25} gives a scale-oblivious $O(\sqrt k\log k)$ guarantee for $T>4k$. The envelope-dependent competitive guarantees of \cite{BGRS26} require a suitable exponent parameter, and their data-driven parameter selection uses offline instances. Our noiseless bounds in Theorems~\ref{thm:W46} and~\ref{thm:W47} instead adapt simultaneously to every admissible $\beta$ on each instance, without reading $m$ or $\beta$ and without using offline instances.

\paragraph{Our results.}
\begin{enumerate}[leftmargin=*]
\item \textbf{Unknown scale is free (Section~\ref{sec:scale}).} A \emph{probe-and-commit} algorithm (sample $s=\lfloor\sqrt k\rfloor$ arms, pull each $\lfloor T/(2s)\rfloor$ times, commit to the arm with the largest probe value) achieves competitive ratio $4\sqrt3\,\sqrt k<7\sqrt k$ for all $T\ge2\lfloor\sqrt k\rfloor$, knowing nothing about $m$. The $\log k$ of~\cite{BR25} is therefore unnecessary. We determine the optimal ratio for every $k,T\ge1$: $\Theta(\sqrt k+k/T)$ (Theorem~\ref{thm:allT}). An anytime version, which reads only $k$, attains this order at \emph{every} horizon (Theorem~\ref{thm:anytimeAll}); the naive doubling schedule, which starts with a phase of length $2\lfloor\sqrt k\rfloor$, does not (Example~\ref{ex:sched}).
\item \textbf{Without noise, every prior is free (Section~\ref{sec:priorfree}).} A \emph{random-marginal probe} algorithm (independent random weights $Z_i$, greedy on the last observed increment divided by $Z_i$, then commit) knows neither $m$ nor $\beta$ and achieves $\E[\ALG_T]\ge\frac{3}{256}\,\OPT_T\,k^{-\beta/(1+\beta)}$ simultaneously for every $\beta$ such that the \emph{optimal arm} satisfies the envelope condition, for $T\ge16k$ (Theorem~\ref{thm:TG}); an independent-inclusion variant extends this to all horizons, giving the optimal $\Theta(k^{\beta/(1+\beta)}+k/T)$ (Theorem~\ref{thm:W46}), and an anytime version needs not even $T$ (Theorem~\ref{thm:W47}). Hence the ``price of not knowing $m,\beta,T$'' is $\Theta(1)$.
\item \textbf{Under noise the price of priors jumps (Section~\ref{sec:noise}).} In the multiplicative noise model of~\cite[App.~C]{BR25}, probe-and-commit retains its guarantees up to the factor $a^2/b$ without knowing the noise level $\varepsilon$: the all-horizon order $\Theta(\sqrt k+k/T)$, and ratio $O(\sqrt k)$ for $T\ge2\lfloor\sqrt k\rfloor$ (their algorithm ``must have knowledge of the value of $\varepsilon$'' and loses $\log k$). But the random-marginal algorithm fails under noise, as do the two increment-based repairs made precise in Section~\ref{sec:noise} (exact counterexamples with unbounded ratio); we do not claim that every conceivable repair fails, and Theorem~\ref{thm:H2L} is what rules out a uniformly constant price. We prove that for any fixed $\varepsilon\in(0,\tfrac12]$, every algorithm knowing neither $m$ nor $\beta$ --- even one that reads $T$ and $\varepsilon$ --- loses a factor $\Omega_\varepsilon(\PhiK(k))$, $\PhiK(k)=\sqrt{(1+\ln k)/\ln(e+\ln k)}$, relative to $k^{\beta/(1+\beta)}$ on some instance at the hard horizons of the construction (an asymptotic statement in $k$; Theorem~\ref{thm:H2L}), and a nested random-permutation probing algorithm matches it (Theorem~\ref{thm:H2U}); knowing $m$ alone (Theorem~\ref{thm:RET}) or $\beta$ alone restores a constant price. The quantifiers are those of the uniform price of adaptation $\phi_\varepsilon(k)$ defined in Section~\ref{sec:model}: one algorithm is chosen first, and then $\beta$, the instance and the noise table are chosen against it.
\end{enumerate}

\paragraph{Why was the logarithm there?}
The algorithms of~\cite{BR25,BGRS26} abandon an arm when its current reward falls below an \emph{absolute} threshold $m\,t/T$ (or $m(t/T)^\alpha$); the scale $m$ must be known, and the treatment of unknown $m$ in~\cite{BR25}---a random dyadic guess over $O(\log k)$ dyadic levels---costs $\log k$. (We describe the technical difference only and make no claim about why this route was taken.) Our \emph{scale-oblivious} algorithms do not compare rewards against a supplied absolute reward scale: probe-and-commit compares probe values \emph{relative to each other} (an $\argmax$ is scale-free) and pays for the missing information with a $\sqrt k$ subsample, whose $1/\sqrt k$ hit probability is exactly the price the lower bound charges anyway. The random-marginal algorithm goes further and turns ``is this arm still worth pulling?'' into a comparison of observed increments weighted by fixed random priorities; its key lemma bounds the expected greedy work spent on decoys by their total value, with no multi-scale union bound.

\paragraph{Organization.} Section~\ref{sec:model} fixes the model and conventions. Section~\ref{sec:scale} treats unknown scale, Section~\ref{sec:priorfree} prior-free algorithms without noise, Section~\ref{sec:noise} the noisy model, and Section~\ref{sec:discussion} deterministic algorithms, comparisons and open problems.

\section{Model and conventions}\label{sec:model}

We follow the specification of~\cite{Pat23,BR25}. An instance consists of $k$ arms; arm $i$ has a reward function $f_i:\{0,1,\dots,T\}\to\R_{\ge0}$ with $f_i(0)=0$, nondecreasing, and with \emph{diminishing returns}:
\[
f_i(t+1)-f_i(t)\le f_i(t)-f_i(t-1)\qquad\text{for all }t\ge1
\]
(discrete concavity). Pulling arm $i$ for the $t$-th time yields $f_i(t)$; the argument is the number of pulls of that arm, not calendar time. The functions are unknown; the algorithm observes only the rewards it collects. The horizon is $T$. The adversary is \emph{oblivious}: all $f_i$ are fixed after $T$ is fixed and before the algorithm runs.

Write $P_i(n)=\sum_{t=1}^n f_i(t)$, $R=\OPT_T=\max_iP_i(T)$, and let $\ast$ denote a \emph{comparator}: an arm attaining the maximum, chosen as part of the statement of each theorem (by default the smallest index). Set $m=f_\ast(T)$ and $M=\max_if_i(T)$. Several arms may be cumulatively optimal with different terminal values $f_i(T)$, so the choice matters and is never left implicit: in the prior-free theorems the comparator is \emph{any} cumulatively optimal arm satisfying the envelope condition below (the algorithm does not read this choice, which only enters the analysis), whereas in Theorem~\ref{thm:RET}, where a scale is supplied to the algorithm, the comparator is the arm that scale refers to. All quantities $m$, $\hat R$ (Section~\ref{sec:noise}) and the envelope hypotheses refer to the same comparator; since $P_\ast(T)=R$ for every admissible choice, the noisy benchmark changes by at most a factor $b/a$ between comparators. Pulling $\ast$ throughout is optimal in hindsight~\cite[Prop.~1]{HKR16}; Lemma~\ref{lem:L1}(6) gives a self-contained proof using only nondecreasing rewards. $\ALG_T$ denotes the algorithm's cumulative reward; expectations are over the algorithm's internal randomness only. An algorithm has \emph{competitive ratio} $\rho$ if $\OPT_T\le\rho\,\E[\ALG_T]$ on every instance~\cite[Def.~1]{BR25}. Instances with $R=0$ make every guarantee trivial and are excluded from ratios. \emph{Computational model.} We idealise the evaluation of the displayed real-valued keys and count exact comparisons. The algorithms with proved positive guarantees use $O(k)$ working storage; standard heap implementations of $\RMP$ and $\SRMP$ use $O(k)$ setup and $O(\log k)$ time per probe, and sorting the weights for $\RET$ takes $O(k\log k)$. The randomness used is $O(k)$ independent draws per phase --- not $O(k)$ random bits --- and Section~\ref{sec:priorfree} records a finite-bit variant. We stress that a small \emph{number of probes} does not mean a small running time: in the sparse regime $\SRMP$ may probe far fewer than $k$ arms, but it includes every arm once $H\ge128k$, and generating its $k$ inclusion variables coordinatewise already costs $\Theta(k)$.

\paragraph{Envelope exponent.} Following~\cite[Def.~3.2]{BGRS26}, for $\beta\in(0,1]$ we say that arm $i$ satisfies $\LE(\beta)$ if the inequality $f_i(t)\ge f_i(T)\,(t/T)^\beta$ holds for all $t\le T$. Every concave nondecreasing $f$ with $f(0)=0$ satisfies $\LE(1)$. Our positive results require the envelope condition \emph{only for the optimal arm}; the lower bounds of~\cite{BGRS26} use instances in which every arm satisfies it, so both directions apply to the class where all arms do. We write
\[
\gamma=\frac{\beta}{1+\beta},\qquad q_\beta=k^{-\gamma},\qquad K_\beta=k^{1/(1+\beta)}=kq_\beta,\qquad Q_\beta(k,T)=\min\{q_\beta,\,T/k\}.
\]

For $\beta\in(0,1]$ let $\mathcal I_\beta(k,T)$ be the class of instances in which at least one cumulatively optimal arm satisfies $\LE(\beta)$ (an upper-envelope class, not the class with envelope exponent exactly $\beta$; when several cumulatively optimal arms exist, the comparator $\ast$ of an instance in $\mathcal I_\beta$ is by convention one satisfying $\LE(\beta)$, and $m$, $\hat R$ and the envelope hypothesis all refer to that same arm, as fixed above), and define the minimax ratio $\rho_\beta(k,T)=\inf_A\sup_{I\in\mathcal I_\beta(k,T),\,R_I>0}R_I/\E[\ALG_T(A,I)]$, the infimum being over the algorithms of the knowledge class under discussion; $\rho(k,T)=\rho_1(k,T)$.

\paragraph{Knowledge.} An algorithm is \emph{scale-oblivious} if it does not read $m$ (nor any estimate of it), \emph{$\beta$-oblivious} if it does not read $\beta$ (nor an upper bound), and \emph{anytime} if it does not read $T$. We always list explicitly what each algorithm reads. Being oblivious to a parameter forbids receiving it as input; it does not forbid estimating quantities from the observations already paid for.

\paragraph{Two adaptation criteria.} For a fixed $\beta$, $\rho_\beta(k,T)$ is a \emph{class-wise} minimax value: the infimum may be attained by a different algorithm for each $\beta$, so a family of bounds ``$\forall\beta\ \exists A_\beta$'' does not by itself produce one algorithm that is near-optimal for all $\beta$ at once. The latter, stronger property --- ``$\exists A\ \forall\beta$'' --- is what we call \emph{simultaneous adaptation}; it is what Theorems~\ref{thm:TG}, \ref{thm:W46} and \ref{thm:H2U} provide, and what Theorem~\ref{thm:H2L} limits under noise. Quantitatively, for a class $\mathfrak A$ of algorithms (specified by what they read) define the \emph{uniform price of adaptation}
\[
\phi_\varepsilon(k)\ =\ \inf_{\{A_T\}_{T\ge16k}\in\mathfrak A}\ \ \sup_{T\ge16k}\ \ \sup_{\beta\in(0,1]}\ \ \sup_{\substack{I\in\mathcal I_\beta(k,T),\,R_I>0\\ \hat f\in\mathcal N_\varepsilon(I)}}\ \frac{\hat R_I\,q_\beta}{\E[\widehat{\ALG}_T(A_T,I,\hat f)]},
\]
where $\mathcal N_\varepsilon(I)$ is the set of admissible noise tables for $I$ (a single table $\hat f=f$ when $\varepsilon=0$, in which case $\hat R_I=R_I$). The order of the quantifiers is essential: the algorithm is chosen before $\beta$, the instance and the noise. For known horizons $\mathfrak A$ consists of families $\{A_T\}$, one algorithm per horizon; for unknown horizons it consists of single anytime policies, evaluated at every $T$. Whether the algorithms may read $\varepsilon$ is stated in each result; our upper bounds on $\phi_\varepsilon$ do not read it, and the lower bound holds even for algorithms that do. Table~\ref{tab:price} is stated in terms of $\phi_\varepsilon$. Ratios are set to $+\infty$ when the numerator is positive and the expected reward is zero.

\paragraph{Scale advice and unknown horizons.} The scale $m=f_\ast(T)$ depends on the horizon. A result of the form ``$m$ known, $T$ unknown'' therefore means: the algorithm does not read $T$, but is given one number $\mu$ --- an advice tied to the horizon at which it will be evaluated --- and the guarantee is asserted at that horizon. It is not the stronger statement that a single policy with a single advice does well at every stopping time simultaneously, and we never claim the latter.

\paragraph{Noise model (Section~\ref{sec:noise}).} In the multiplicative model of~\cite[App.~C]{BR25}, ``when we pull an arm, instead of getting the exact value of $f_i(t_i)$, we get a reward value $\hat f_i(t_i)\in[(1-\varepsilon)f_i(t_i),(1+\varepsilon)f_i(t_i)]$''; the algorithm observes and earns $\hat f$. The noisy values are fixed by the oblivious adversary before the run, for every $(i,t)$, with no further structure (they may be non-monotone and non-concave). We compare, as in~\cite{BR25}, with the policy that plays the best uncorrupted arm: $\hat R=\sum_{t\le T}\hat f_\ast(t)$, which satisfies $aR\le\hat R\le bR$ with $a=1-\varepsilon$, $b=1+\varepsilon$, and $\hat R\ge\frac ab\max_\pi\hat V(\pi)$. Throughout we assume $\varepsilon\in[0,\tfrac12]$; this is a convention of the present paper (the additional restriction $\varepsilon\le\tfrac12$ is not imposed in~\cite{BR25}), used only by the algorithm $\RET$ of Section~\ref{sec:noise}; the probe-and-commit and $\NPA$ bounds hold for all $\varepsilon<1$.

\paragraph{Versions.} All quotations from~\cite{BR25} follow the ALT 2025 proceedings version (PMLR~272), whose appendix on noisy rewards is not contained in the 12-page arXiv v1 of April 2024; quotations from~\cite{BGRS26} follow arXiv v2 (May 2026).

\paragraph{Prior results used.} We use the following statements from the literature only as \emph{comparisons} or as \emph{lower bounds}; all upper bounds in this paper are self-contained.
\begin{itemize}[leftmargin=*]
\item \cite[Thm.~3, Cor.~4]{BR25}: for any randomized algorithm there is an instance with approximation factor at least $\sqrt k/3$ (the hard distribution has $f^*(T)=1$, so the bound holds even when the scale is known).
\item \cite[Thm.~8]{BR25}: let $\tau=T-k$ and let the supplied scale parameter $\mu$ satisfy $\mu\in[f^*(\tau)/c_2,f^*(\tau)]$ for a constant $c_2>1$; for $T\ge2k$, their random round robin run with internal horizon parameter $\tau$ and scale parameter $\mu$ earns at least $\OPT_T/(8c_2\sqrt k)$. Here $\mu$ is the cited algorithm's parameter and is not our $m=f^*(T)$. \cite[Thm.~11, Lemma~12]{BR25}: with $m$ unknown and $T>4k$, $\Omega(\OPT/(\sqrt k\log k))$; with $T$ unknown as well, $\OPT_T/(8192\sqrt k\log(128k))$.
\item \cite[Thm.~3.5]{BGRS26}: with $m$ known and $T\ge2k$, $\mathrm{PTRR}_\alpha$ for $\alpha\in(\beta_I,1)$ has ratio $O(k^{\alpha/(\alpha+1)})$. \cite[Thm.~3.6]{BGRS26}: for every algorithm and $T$ sufficiently large there is an instance with envelope exponent $\beta$ and $\E[\ALG_T]/\OPT_T\le C_\beta k^{-\beta/(\beta+1)}$ (Theorem~B.14 there gives the explicit condition $T\ge2[\beta(\beta+1)]^{1/(\beta+1)}k^{1/(\beta+1)}$; as printed it admits, e.g., $k=2$, $\beta=1/8$, $T=1$, where an instance with envelope exponent exactly $\beta$ cannot exist, so its domain needs a mild restriction; our Theorem~\ref{thm:TD}(i) is self-contained and unaffected). Section~3.2 of \cite{BGRS26} also removes scale knowledge at an $O(\log k)$ cost, and with both the scale and the horizon unknown \cite[Thm.~B.15]{BGRS26} states an $O(k^{\alpha/(\alpha+1)}\log k)$ guarantee for $\beta_I\in(0,1)$, $\alpha\in(\beta_I,1)$ and $T>4k$; that guarantee still requires choosing a suitable exponent parameter $\alpha$, and what we remove is both the logarithmic overhead and the need to select such a parameter from prior information. That unknown-horizon result uses a doubling bookkeeping step in its proof whose displayed inequality relating the final time to the last completed round appears to need a small local adjustment; the order of its conclusion is not affected and we make no claim about it.
\item \cite{Pat23} studies deterministic horizon-unaware improving bandits and states an $O(k)$ competitive guarantee. We use its deterministic lower-bound construction as background, not the displayed algorithm's literal all-horizon upper bound. The lower bound restated in Section~\ref{sec:discussion} uses only determinism and therefore also covers horizon-aware algorithms. The horizon-qualified deterministic bounds used here are proved independently in that section. Some horizon qualification is essential --- the deterministic minimax ratio is unbounded for $T<k$ and equals $k(k+1)/2$ at $T=k$ (Section~\ref{sec:discussion}) --- but we do not claim that $2k$ is a sharp boundary, and the qualification is ours rather than attributed to~\cite{Pat23}. The following observation concerns Appendix~C.1 of arXiv:2208.09254v1 (19 August 2022). Its deterministic lower-bound argument selects the instance after fixing the deterministic algorithm, and the remark that follows does not justify extending that argument to the randomized competitive ratio $R/\E[\ALG]$. At $T=k$ every instance of the family displayed there is nondecreasing, has diminishing returns and may be taken to vanish at $0$, so Theorem~\ref{thm:anytimeAll} below applies to each member and gives ratio at most $1536(\sqrt k+1)$, which is below the asserted $k/2$ once $k\ge2^{24}$. The asserted randomized extension of that construction is therefore incompatible with our guarantee. The comparison is restricted to that family and does not rely on any containment of the full instance class of~\cite{Pat23}.
\end{itemize}

\section{Unknown scale is free}\label{sec:scale}

\subsection{Consequences of discrete concavity}

\begin{lemma}[basic facts]\label{lem:L1}
Let $f$ be nondecreasing, discretely concave, $f(0)=0$, and $P(n)=\sum_{t\le n}f(t)$. Then:
\begin{enumerate}[label=(\arabic*),leftmargin=*]
\item $f(t)\ge\frac tT f(T)$ for $1\le t\le T$;
\item $P(n)\ge f(T)\frac{n(n+1)}{2T}$ for $0\le n\le T$; more precisely $P(u)\ge\frac{u(u+1)}{T(T+1)}P(T)\ge(u/T)^2P(T)$ for $1\le u\le T$ (a ratio form is available when $P(T)>0$);
\item $R\le mT$;
\item if $\lceil T/2\rceil\le r\le T$ then $P_\ast(r)\ge R/4$;
\item $M\le\frac{2T}{T+1}m\le2m$ and $R\ge\frac{T+1}2M$;
\item for any allocation $\sum_in_i=T$ of pulls, $\sum_iP_i(n_i)\le R$.
\end{enumerate}
\end{lemma}
\begin{proof}
Let $\Delta_t=f(t)-f(t-1)$; then $\Delta_1\ge\Delta_2\ge\dots\ge\Delta_T\ge0$ and $f(t)/t$ is the average of $\Delta_1,\dots,\Delta_t$, hence nonincreasing in $t$; this gives (1) and, summing, the first form of (2). For the second form write $P(n)=\sum_{t\le n}t\cdot(f(t)/t)$; the weighted average $P(n)/\sum_{t\le n}t$ of the nonincreasing sequence $f(t)/t$ with weights $t$ is nonincreasing in $n$ (appending a smaller term cannot raise the average), so $P(u)/P(T)\ge\frac{u(u+1)/2}{T(T+1)/2}$. (3) is monotonicity. (4) follows from (2) with $u=r\ge T/2$: $\frac{r(r+1)}{T(T+1)}\ge\frac{(T/2)(T/2+1)}{T(T+1)}\ge\frac14$. For (5), (2) with $n=T$ gives $P_i(T)\ge f_i(T)\frac{T+1}2$, so $M\frac{T+1}2\le R\le mT$. For (6), the prefix averages $P_i(n)/n$ of the nondecreasing sequence $f_i$ are nondecreasing, so $P_i(n_i)\le\frac{n_i}TP_i(T)\le\frac{n_i}TR$ and summing over $i$ gives $R$.
\end{proof}

\subsection{Probe-and-commit}

\begin{algorithm}[h]
\caption{$\PC(s,\tau)$ (probe-and-commit). Inputs: $k$, $T$, integers $1\le s\le k$, $\tau\ge1$ with $s\tau\le\lfloor T/2\rfloor$.}
\begin{algorithmic}[1]
\State Draw a uniformly random set $S$ of $s$ distinct arms; pull each arm of $S$ exactly $\tau$ times.
\State Let $c=\argmax_{i\in S}f_i(\tau)$ (any fixed tie-breaking rule).
\State Pull $c$ for the remaining $r=T-s\tau\ge\lceil T/2\rceil$ steps.
\end{algorithmic}
\end{algorithm}

The algorithm reads only $k$ and $T$. The analysis rests on a negative-correlation inequality for the random sample.

\begin{lemma}[sample negative correlation]\label{lem:L2}
Fix $F\subseteq[k]\setminus\{\ast\}$ with $h=|F|$, and let $A=\{\ast\in S\}$, $B=\{S\cap F\ne\emptyset\}$. Then $\Pr(A\cap\neg B)\ge\frac sk\Pr(\neg B)$.
\end{lemma}
\begin{proof}
If $s>k-h$ then $\Pr(\neg B)=0$. Otherwise, conditioned on $\neg B$, $S$ is a uniform $s$-subset of the $k-h$ arms outside $F$, a set containing $\ast$, so $\Pr(A\mid\neg B)=\frac s{k-h}\ge\frac sk$.
\end{proof}

\begin{lemma}[core guarantee]\label{lem:L3}
For every instance and every admissible $(s,\tau)$, with $r=T-s\tau$,
\[
\E[\ALG_T]\ \ge\ R\cdot\min\Big\{\frac{r\tau}{T^2},\ \frac{s\,r(r+1)}{2kT^2}\Big\}\ \ge\ R\cdot\min\Big\{\frac{\tau}{2T},\ \frac{s}{8k}\Big\}.
\]
\end{lemma}
\begin{proof}
Let $F=\{i\ne\ast: f_i(\tau)\ge f_\ast(\tau)\}$, a set determined by the instance, and let $A,B$ be as in Lemma~\ref{lem:L2}. On $B$ (some ``fooling'' arm was sampled) the committed arm satisfies $f_c(\tau)\ge f_\ast(\tau)\ge m\tau/T$ by Lemma~\ref{lem:L1}(1), and since $f_c$ is nondecreasing the commit phase earns at least $r\,m\tau/T\ge R\,r\tau/T^2$ by Lemma~\ref{lem:L1}(3). On $A\cap\neg B$ the optimal arm is the unique maximizer, so $c=\ast$ and the commit phase earns $\sum_{j=1}^rf_\ast(\tau+j)\ge\sum_{j=1}^rf_\ast(j)\ge R\frac{r(r+1)}{2T^2}$ by Lemma~\ref{lem:L1}(2),(3). The two events are disjoint; with $p=\Pr(B)$ and Lemma~\ref{lem:L2}, $\E[\ALG_T]/R\ge\frac{r\tau}{T^2}p+\frac{s\,r(r+1)}{2kT^2}(1-p)$, which is at least the minimum of the two coefficients. The second form uses $r\ge T/2$. Probe rewards are nonnegative and were dropped.
\end{proof}

\begin{theorem}[unknown scale]\label{thm:TA}
Let $k\ge2$, $T\ge2\lfloor\sqrt k\rfloor$, $s=\lfloor\sqrt k\rfloor$, $\tau=\lfloor T/(2s)\rfloor$. Then $\PC(s,\tau)$ satisfies
\[
\E[\ALG_T]\ \ge\ \frac{s}{8k}R\ \ge\ \frac{R}{16\sqrt k},\qquad\text{and in fact}\qquad \E[\ALG_T]\ \ge\ \frac{R}{4\sqrt3\,\sqrt k}.
\]
For $k=1$ the ratio is $1$. The algorithm uses no information about $m$.
\end{theorem}
\begin{proof}
Since $\lfloor x\rfloor\ge x/2$ for $x\ge1$, $\tau\ge T/(4s)$ and $s\tau\le T/2$; since $s^2\le k$, $\frac1{8s}\ge\frac s{8k}$; and $\lfloor\sqrt k\rfloor\ge\sqrt k/2$. Lemma~\ref{lem:L3} gives the first bound. For the sharper constant, use the two branches of Lemma~\ref{lem:L3} with the refined Lemma~\ref{lem:L1}(2) and the following inequality: for $1\le\tau\le(T+1)/2$, $f(\tau)/P(T)\ge\frac{2\tau}{T(T+1)}$. (Write $f=\sum_{d=1}^Tc_d\min\{t,d\}$ with $c_d=\Delta_d-\Delta_{d+1}\ge0$, $\Delta_{T+1}:=0$; for each basis function $\min\{t,d\}$ the ratio of its prefix sum $dT-d(d-1)/2$ to its value at $\tau$ is at most $T(T+1)/(2\tau)$ when $d\ge\tau$ (maximized at $d=T$) and equals $T-(d-1)/2\le T\le T(T+1)/(2\tau)$ when $d<\tau$; a nonnegative combination preserves the bound.) Hence $\E[\ALG_T]/R\ge\min\{\frac{2r\tau}{T(T+1)},\frac sk\frac{r(r+1)}{T(T+1)}\}$. With $r\ge T/2$ and $\tau\ge T/(4s)$ the first term is at least $\frac{T}{4s(T+1)}\ge\frac1{6\sqrt k}$ (as $T\ge2$), the second at least $\frac s{4k}\ge\frac1{4\sqrt3\sqrt k}$ because $k\le s^2+2s\le3s^2$. The smaller is the second.
\end{proof}

\begin{remark}
Theorem~\ref{thm:TA} should be compared with \cite[Thm.~11]{BR25}, which requires $T>4k$ and loses $\log k$. The gap to the lower bound $\sqrt k/3$ of \cite{BR25} is a constant factor $4\sqrt3\cdot3<21$. Extrapolation-based thresholds cannot be made scale-free by simply halving a guessed threshold: for the instance $T=k^2$, $f_\ast(t)=3t/(4T)$ and $k-1$ decoys $\min(t,k)/T$, a threshold $\mu=1$ rejects $\ast$ after one pull and each decoy absorbs $k+1$ pulls, exhausting the horizon with ratio $\Theta(k)$; guessing the level at random is exactly the $\log k$ of \cite{BR25}.
\end{remark}

\subsection{Lower bounds and the full horizon dependence}

The following ``delayed-reveal'' lemma is a restatement, with integer bookkeeping, of the hard-distribution argument of \cite{BR25} (see also \cite[Lemma~B.9]{BGRS26}, of which it is a variant); we include it because it yields the constant $3$ without divisibility conditions and the short-horizon term.

\begin{lemma}[delayed reveal]\label{lem:L4}
Let $g$ be nonnegative, nondecreasing, discretely concave with $g(0)=0$, and $0\le d\le T$ an integer. Consider the distribution $\mathcal D_{g,d}$ over instances in which a uniformly random arm $G$ has reward $g$ and every other arm has $h_d(t)=g(\min\{t,d\})$; let $R_g=\sum_{t\le T}g(t)$. For every randomized algorithm $A$,
\[
\E_{I\sim\mathcal D_{g,d}}\E[\mathrm{reward}(A,I)]\ \le\ T\,g(d)+\frac{\lfloor T/(d+1)\rfloor}{k}\,R_g ,
\]
the inner expectation being over the internal randomness of $A$. Consequently, for every $A$ there is an instance in the support of $\mathcal D_{g,d}$ on which $\E[\mathrm{reward}]$ is at most this quantity, while $\OPT_T=R_g$ on every instance of the support.
\end{lemma}
\begin{proof}
Fix the algorithm's random seed. Run it against the virtual instance in which every arm is $h_d$; let $D$ be the set of arms pulled at least $d+1$ times there, so $|D|\le\lfloor T/(d+1)\rfloor$. Let $E$ be the event that, in the real run, the good arm is pulled for the $(d+1)$-st time within the horizon (this is not required to be observationally recognizable). Before that pull the real and virtual runs coincide, hence $E$ holds iff $G\in D$, and $\Pr_G(E)\le|D|/k$. On $E^c$ every pull yields at most $g(d)$; on $E$ the total reward is at most $R_g$ by Lemma~\ref{lem:L1}(6) applied to the real instance. Averaging over the seed gives the displayed bound, and the last sentence is the usual averaging (Yao's principle \cite{Yao77}) step: a random variable is somewhere below its mean.
\end{proof}

\begin{corollary}\label{cor:L4}
For every randomized algorithm there is an instance on which its competitive ratio is at least
(a) $\sqrt k/3$, for all $k,T\ge1$ (take $g(t)=t/T$, $d=\lfloor T/\sqrt k\rfloor$); and
(b) $k/T$ (take $g(t)=t/T$, $d=0$: one linear arm and $k-1$ arms identically zero).
\end{corollary}
\begin{proof}
(a) $R_g=\frac{T+1}2$; $2d/(T+1)<2/\sqrt k$ and $T/(k(d+1))<1/\sqrt k$ since $d+1>T/\sqrt k$; if $T<\sqrt k$ then $d=0$ and the construction remains legal. (b) is the case $d=0$ of Lemma~\ref{lem:L4}, where $g(0)=0$ kills the first term. Since \cite{BR25} phrase reward functions as ``monotone increasing'', note that the zero decoys may be replaced by arms of slope $\eta\to0$ without changing the order of magnitude. Both families consist of linear-plus-truncation curves; they belong to $\mathcal I_1$, and for $T\ge2$ the linear good arm violates $\LE(\beta)$ whenever $\beta<1$, so these instances may not be used for the corresponding $\rho_\beta$ (see Theorem~\ref{thm:TD}(ii) for the version that stays inside $\mathcal I_\beta$). At $T=1$ every admissible arm satisfies every $\LE(\beta)$, so the restriction is vacuous there.
\end{proof}

\begin{theorem}[all horizons]\label{thm:allT}
For all integers $k,T\ge1$, the minimax ratio of scale-oblivious algorithms satisfies
\[
\frac14\Big(\sqrt k+\frac kT\Big)\ \le\ \rho(k,T)\ \le\ 24\Big(\sqrt k+\frac kT\Big).
\]
\end{theorem}
\begin{proof}
Upper bound: for $T=1$ pull a uniformly random arm (ratio $\le k$). For $T\ge2$ run $\PC$ with $s=\min\{\lfloor\sqrt k\rfloor,\lfloor T/2\rfloor\}$ and $\tau=\lfloor T/(2s)\rfloor$ (so $s\ge1$, $\tau\ge1$ and $s\tau\le\lfloor T/2\rfloor$); the computation of Theorem~\ref{thm:TA} applied to Lemma~\ref{lem:L3} gives $\E[\ALG_T]\ge\frac s{8k}R$, and $\lfloor T/2\rfloor\ge T/3$ for $T\ge2$, so $R/\E[\ALG_T]\le8k/s\le\max\{16\sqrt k,24k/T\}$. Note that this is a corollary of Lemma~\ref{lem:L3} with truncated parameters; the hypothesis $T\ge2\lfloor\sqrt k\rfloor$ of Theorem~\ref{thm:TA} is not needed. Lower bound: Corollary~\ref{cor:L4} gives $\max\{\sqrt k/3,k/T\}\ge\frac14(\sqrt k+k/T)$.
\end{proof}

\subsection{Unknown horizon}

\begin{lemma}[conditioning on history]\label{lem:L5}
Suppose that at the start of a phase $t_0=\sum_in_i$ pulls have been made (arm $i$ has been pulled $n_i$ times), and the phase runs $\PC(s,\tau)$ with a fresh independent sample, $1\le s\le k$, and commit length $r$, where $t_0+s\tau+r\le T$ (so every value $f_i(n_i+\tau)$ used below is within the domain). Then $\E[\text{phase reward}\mid\text{history}]\ge R\min\{\frac{r\tau}{T^2},\frac{s\,r(r+1)}{2kT^2}\}$.
\end{lemma}
\begin{proof}
Define $F_H=\{i\ne\ast: f_i(n_i+\tau)\ge f_\ast(n_\ast+\tau)\}$, a fixed set given the history; Lemma~\ref{lem:L2} applies to the fresh sample. On $B$ the committed arm's probe value is at least $f_\ast(n_\ast+\tau)\ge f_\ast(\tau)$; on $A\cap\neg B$ the committed arm is $\ast$ and earns $\sum_{j\le r}f_\ast(n_\ast+\tau+j)\ge\sum_{j\le r}f_\ast(j)$. Conclude as in Lemma~\ref{lem:L3}.
\end{proof}

The obvious doubling schedule, with phases $2s,4s,8s,\dots$, works as soon as the horizon exceeds the first phase.

\begin{theorem}[unknown horizon, $T\ge2\lfloor\sqrt k\rfloor$]\label{thm:anytime}
Let $s=\lfloor\sqrt k\rfloor$ and run phases of lengths $L_j=2s\cdot2^j$, $j=0,1,\dots$, each running $\PC(s,L_j/(2s))$ with a fresh sample and without resetting pull counts. The algorithm reads only $k$, and for every $T\ge2s$, $\E[\ALG_T]\ge R/(256\sqrt k)$.
\end{theorem}
\begin{proof}
Let $L$ be the length of the last completed phase; $2L-2s\le T<4L-2s$, so $L>T/4$. Lemma~\ref{lem:L5} with $\tau=L/(2s)$, $r=L/2$ gives $\E\ge R(L/T)^2\min\{\frac1{4s},\frac s{8k}\}=R(L/T)^2\frac s{8k}\ge R/(256\sqrt k)$. Compare \cite[Lemma~12]{BR25}: $\OPT_T/(8192\sqrt k\log(128k))$ for $T>4k$.
\end{proof}

This schedule guarantees nothing before its first phase ends, and that restriction is real, not an artefact of the analysis.

\begin{example}[the schedule, not the horizon, is the obstruction]\label{ex:sched}
Let $n\ge2$, $k=n^2$, $T=n$, $f_\ast(t)=t$ and all other arms identically zero. The algorithm of Theorem~\ref{thm:anytime} has $s=n$ and first phase $2n>T$: it spends the whole horizon on the $n$ initial probes of the sampled arms and never commits, so $\E[\ALG_T]=\Pr(\ast\in S)\cdot f_\ast(1)=1/n$ while $R=n(n+1)/2$; its ratio is $n^2(n+1)/2$. Yet the optimal ratio at these parameters is $\Theta(\sqrt k+k/T)=\Theta(n)$ by Theorem~\ref{thm:allT}, attained by $\PC$ with truncated parameters when $T$ is known. (On a parameter range, the all-horizon minimax value is $O(\sqrt k)$ if and only if $T=\Omega(\sqrt k)$ uniformly on that range; equivalently, an $O(\sqrt k)$ guarantee is ruled out exactly when $\sqrt k/T$ is unbounded on the range. In particular $T=o(\sqrt k)$ rules it out, while the range $T=\Omega(\sqrt k)$, which still contains parameters with $T<2s$, does not.) The next algorithm removes the restriction by starting at length~$1$.
\end{example}

\begin{algorithm}[h]
\caption{$\APC$ (anytime probe-and-commit). Reads only $k$; it may be stopped at any time, and its true pull counts are never reset.}
\begin{algorithmic}[1]
\State If $k=1$: always pull the unique arm. Otherwise run phases $j=0,1,2,\dots$ of planned lengths $H_j=2^j$.
\State \textbf{Phase of length $H$:} if $H=1$, pull one uniformly random arm once.
\State Otherwise let $s=\min\{\lfloor\sqrt k\rfloor,\lfloor H/2\rfloor\}$, $\tau=\lfloor H/(2s)\rfloor$ and $r=H-s\tau$.
\State Draw a fresh uniformly random set $S$ of $s$ distinct arms and pull each of them $\tau$ times within this phase.
\State Commit the remaining $r$ pulls of the phase to the arm of $S$ whose $\tau$-th observation \emph{in this phase} is largest (ties broken by index).
\end{algorithmic}
\end{algorithm}

\begin{theorem}[unknown horizon, all horizons]\label{thm:anytimeAll}
For all integers $k,T\ge1$, $\APC$ satisfies
\[
\E[\ALG_T]\ \ge\ \frac{R}{1536\,(\sqrt k+k/T)} .
\]
\end{theorem}
\begin{proof}
For $k=1$ the algorithm earns $R$. Let $k\ge2$. Each phase of length $H\ge2$ is admissible: $s\ge1$, $H/(2s)\ge1$ so $\tau\ge1$, and $s\tau\le H/2$, whence $r\ge\lceil H/2\rceil$. Let $H=2^j$ be the length of the last phase completed within the horizon. The first $j+1$ phases occupy $2H-1$ steps, so $2H-1\le T$; if the next phase also completed we would have $4H-1\le T$, so $T<4H-1$. Hence $x:=H/T>1/4$.

\emph{Case $H\ge2$.} The phase satisfies the hypotheses of Lemma~\ref{lem:L5} (fresh sample, and $t_0+s\tau+r=t_0+H\le T$ because the phase is completed), so, using $r\ge H/2$ and $\tau\ge H/(4s)$ (valid since $H/(2s)\ge1$ and $\lfloor y\rfloor\ge y/2$ for $y\ge1$),
\[
\E[\ALG_T]\ \ge\ \E[\text{this phase}]\ \ge\ R\min\Big\{\frac{r\tau}{T^2},\frac{s\,r(r+1)}{2kT^2}\Big\}\ \ge\ \frac{Rx^2}{8}\min\Big\{\frac1s,\frac sk\Big\},
\]
all other phases contributing nonnegative rewards. If $s=\lfloor\sqrt k\rfloor$ then $1/s\ge1/\sqrt k$ and $s/k\ge1/(2\sqrt k)$, so the last factor is at least $1/(2\sqrt k)$ and $\E[\ALG_T]\ge Rx^2/(16\sqrt k)$. If instead $s=\lfloor H/2\rfloor<\lfloor\sqrt k\rfloor$ then $s<\sqrt k$, so $\min\{1/s,s/k\}=s/k\ge H/(3k)$ (as $\lfloor H/2\rfloor\ge H/3$ for $H\ge2$), and $\E[\ALG_T]\ge Rx^2H/(24k)=Rx^3T/(24k)$. Since $x\le1$, both cases give
\[
\E[\ALG_T]\ \ge\ Rx^3\min\Big\{\frac1{16\sqrt k},\frac T{24k}\Big\}\ \ge\ \frac{Rx^3}{24\,(\sqrt k+k/T)}\ \ge\ \frac{R}{1536\,(\sqrt k+k/T)},
\]
where the middle inequality is checked separately for the two branches ($24(\sqrt k+k/T)\ge16\sqrt k$; and $(\sqrt k+k/T)\,T/k\ge1$), and the last uses $x^3>1/64$.

\emph{Case $H=1$.} Then $1\le T<3$. The completed phase of length $1$ pulls a uniformly random arm, so $\E[\ALG_T]\ge\frac1k f_\ast(1)\ge\frac{2R}{kT(T+1)}$ by Lemma~\ref{lem:L1}(2), which is $R/k$ for $T=1$ and $R/(3k)$ for $T=2$; both exceed $R/(1536(\sqrt k+k/T))\le R/(768k)$.
\end{proof}

\subsection{Known envelope exponent}

\begin{theorem}[known $\beta$, unknown scale]\label{thm:TD}
Let $\beta\in(0,1]$ be known and suppose the comparator arm satisfies $\LE(\beta)$. For $k\ge2$, $T\ge2\lceil K_\beta\rceil$, $s=\lceil K_\beta\rceil$, $\tau=\lfloor T/(2s)\rfloor$, $\PC(s,\tau)$ satisfies $\E[\ALG_T]\ge R/(16k^{\gamma})$. Moreover:
\begin{enumerate}[label=(\roman*),leftmargin=*]
\item for every $k\ge2$ and $T\ge2$ there is an explicit distribution over instances, all of whose arms satisfy $\LE(\beta)$ and whose envelope exponent equals $\beta$, on which \emph{every} (possibly randomized) algorithm has expected reward at most $(\beta+2)k^{-\gamma}\OPT_T$; hence for every algorithm some instance of that family forces ratio at least $k^\gamma/(\beta+2)$, and $\rho_\beta(k,T)\ge k^\gamma/(\beta+2)$;
\item for all $T\ge1$, $\frac{k^\gamma+k/T}{\beta+3}\le\rho_\beta(k,T)\le24(k^\gamma+k/T)$;
\item with $T$ unknown: for $k\ge2$ and $s=\lceil K_\beta\rceil$, independent $\PC$ phases of lengths $2s\cdot2^j$ ($j=0,1,\dots$) without resetting pull counts give $\E[\ALG_T]\ge R/(128k^\gamma)$ for every $T\ge2s$ (for $k=1$ pull the only arm).
\end{enumerate}
\end{theorem}
\begin{proof}
In Lemma~\ref{lem:L3} replace the bound $f_\ast(\tau)\ge m\tau/T$ by $f_\ast(\tau)\ge m(\tau/T)^\beta$, which gives
\[
\E[\ALG_T]/R\ \ge\ \min\Big\{\frac1{2(4s)^\beta},\ \frac s{8k}\Big\};
\]
since $K_\beta\le s\le2K_\beta$, $\frac1{2(4s)^\beta}\ge\frac1{2\cdot8^\beta k^\gamma}\ge\frac1{16k^\gamma}$ and $\frac s{8k}\ge\frac1{8k^\gamma}$.

(i) Apply Lemma~\ref{lem:L4} with $g(t)=(t/T)^\beta$ (discretely concave: its increments are integrals of the decreasing function $\beta x^{\beta-1}$) and $d=\lfloor T/K_\beta\rfloor$. Every arm of the support satisfies $\LE(\beta)$: the good arm with equality, and $h_d(t)=g(\min\{t,d\})$ because $h_d(t)=g(t)\ge h_d(T)(t/T)^\beta$ for $t\le d$ and $h_d(t)=h_d(T)\ge h_d(T)(t/T)^\beta$ for $t>d$. The good arm violates $\LE(\beta')$ for every $\beta'<\beta$ at $t=1$ (here $T\ge2$), so the envelope exponent of each instance is exactly $\beta$. With $R_g=\sum_{t\le T}(t/T)^\beta\ge T/(\beta+1)$, Lemma~\ref{lem:L4} bounds the expected reward, averaged over the hard distribution, by $(\beta+1)(d/T)^\beta R_g+\frac{T}{k(d+1)}R_g\le(\beta+2)k^{-\gamma}R_g$; the last sentence of Lemma~\ref{lem:L4} converts this into a fixed instance for each algorithm. This reuses the power-curve/truncation construction of \cite[Thm.~3.6]{BGRS26}, with Lemma~\ref{lem:L4} supplying a self-contained integer version valid for all $T\ge2$; \cite[Thm.~B.14]{BGRS26} states a different explicit threshold, and neither threshold dominates the other for all $\beta,k$.

(ii) For $T=1$ pull a uniformly random arm; the one-positive-arm distribution of Lemma~\ref{lem:L4} with $d=0$ gives the matching lower bound, so $\rho_\beta(k,1)=k$ and both displayed bounds hold. Assume henceforth $T\ge2$. Upper bound: run $\PC$ with $s=\min\{\lceil K_\beta\rceil,\lfloor T/2\rfloor\}$ and $\tau=\lfloor T/(2s)\rfloor$, so that $s\ge1$ and $\tau\ge1$. Lemma~\ref{lem:L3} with $f_\ast(\tau)\ge m(\tau/T)^\beta$ gives $\E[\ALG_T]/R\ge\min\{\frac1{2(4s)^\beta},\frac s{8k}\}$; using $s\le2K_\beta$ in the first term and $s\ge\min\{K_\beta,T/3\}$ in the second (valid since $\lfloor T/2\rfloor\ge T/3$ for $T\ge2$) this is at least $\min\{\frac{q_\beta}{16},\frac T{24k}\}\ge\frac1{24(k^\gamma+k/T)}$. Lower bound: besides (i) we need a $k/T$ bound \emph{inside} $\mathcal I_\beta$, so we cannot invoke Corollary~\ref{cor:L4}(b), whose good arm is linear and violates $\LE(\beta)$ for $\beta<1$. Instead apply Lemma~\ref{lem:L4} with $g(t)=(t/T)^\beta$ and $d=0$: the good arm satisfies $\LE(\beta)$ with equality, the other arms are identically zero and satisfy every $\LE(\beta')$, and $g(0)=0$ leaves $\E[\mathrm{reward}]\le\frac Tk R_g$, i.e.\ ratio at least $k/T$. Finally $\max\{k^\gamma/(\beta+2),\,k/T\}\ge(k^\gamma+k/T)/(\beta+3)$.

(iii) As in Theorem~\ref{thm:anytime}: in the last completed phase $L>T/4$ the probe depth is exactly $\tau=L/(2s)$ and $r=L/2$ (no rounding); using $f_\ast(\tau)\ge m(\tau/T)^\beta$ and $\sum_{j\le r}f_\ast(j)\ge m\,r^{1+\beta}/((1+\beta)T^\beta)$ in Lemma~\ref{lem:L5}, the phase earns at least $R(L/T)^{1+\beta}\min\{\frac{1}{2(2s)^\beta},\frac{s}{2^{1+\beta}(1+\beta)k}\}\ge R(L/T)^{1+\beta}/(8k^\gamma)$, and $(L/T)^{1+\beta}\ge1/16$ gives $128$.
\end{proof}

\begin{remark}\label{rem:classwise}
Part (ii) is a \emph{class-wise} statement: for each $\beta$ separately, $\rho_\beta$ compares algorithms that may depend on $\beta$. It does not by itself produce one algorithm that is simultaneously near-optimal for all $\beta$; that stronger property is the subject of Section~\ref{sec:priorfree}.
\end{remark}

\begin{remark}[maximum single pull]
For the objective of maximizing the largest single reward (\cite[App.~B]{BR25}), the last pull of $\PC$ is its maximum observed reward and $\E[\max]\ge\E[\ALG_T]/T\ge M/(32\sqrt k)$ by Lemma~\ref{lem:L1}(5); the transfer argument is that of \cite[App.~B]{BR25}; the resulting guarantee (no scale knowledge, no logarithm, $T\ge2\lfloor\sqrt k\rfloor$) is a direct consequence of Theorem~\ref{thm:TA} and is not identical to a statement in \cite{BR25}.
\end{remark}

\section{Without noise, every prior is free}\label{sec:priorfree}

The $\beta$-tuned algorithms of Theorem~\ref{thm:TD} use $\beta$ to choose the probe depth. We now remove this requirement. Throughout this section the instance is noiseless and we assume only that \emph{the optimal arm} satisfies $\LE(\beta)$; the other arms need only satisfy the model assumptions.

\subsection{Random-marginal probing}

\begin{algorithm}[h]
\caption{$\RMP(H)$ (random-marginal probe-and-commit). Input: $k$ and a length $H\ge16k$. Reads neither $m$ nor $\beta$.}
\begin{algorithmic}[1]
\State Draw independent $Z_1,\dots,Z_k\sim\mathrm{Unif}(0,1)$. Let $N=\lfloor H/2\rfloor$ (probe budget) and $r=H-N$ (commit budget).
\State Pull every arm once (these $k$ pulls count towards $N$); set $n_i=1$ and $\delta_i(1)=g_i(1)$.
\State For the remaining $N-k$ probes: pull the arm maximizing the lexicographic key $\big(\delta_i(n_i)/Z_i,\ 1/Z_i\big)$, where $\delta_i(n_i)=g_i(n_i)-g_i(n_i-1)$ is the \emph{last observed} increment; update $n_i$ and $\delta_i$.
\State Commit the $r$ remaining pulls to the arm maximizing $\big(g_i(n_i),1/Z_i\big)$.
\end{algorithmic}
\end{algorithm}
Here $g_i$ is the curve observed in the current run ($g_i=f_i$ in a single run; for a phase starting after $a_i$ earlier pulls, $g_i(t)=f_i(a_i+t)$ with $g_i(0)=0$). Residual ties (probability zero) are broken by arm index. A finite-randomness version with $Z_i\sim\mathrm{Unif}\{1/M,\dots,1\}$, $M=2^{\lceil\log_2(32k)\rceil}$, and index tie-breaking changes the constants below to $9/1024$ and $9/16384$. For grid-valued weights and $x\ge0$ one has $\Pr(Z\le x)=\min\{1,\lfloor Mx\rfloor/M\}\le\min\{1,x\}$, so the upper bound in Lemma~\ref{lem:G2} is unchanged, and the monotone-coupling argument holds verbatim on the discrete ordered support. For the lower direction, $d\le H/4$ forces the threshold parameter $p=\min\{1,H/(32kd)\}$ to satisfy $p\ge\frac1{8k}$, so with $M\ge32k$ we get $\Pr(Z\le p)\ge p-\frac1M\ge\frac34p$. The discretisation therefore costs at most a factor $\frac34$ in the positive guarantees, which is what the two constants record.

\begin{lemma}[merge structure and monotone coupling]\label{lem:G1}
For each arm consider the finite sequence of items $Q_{i,t}=(\delta_i(t)/Z_i,\,1/Z_i,\,-i,\,-t)$, $1\le t\le H$, which is nonincreasing in $t$ by discrete concavity. The probe phase of $\RMP$ consumes exactly the $N-k$ largest items of the union in this total order, and $n_i=1+\#\{\text{items of arm }i\text{ consumed}\}$. Consequently, for fixed curves and fixed $Z_{-\ast}$, decreasing $Z_\ast$ does not increase any $n_i$, $i\ne\ast$; in particular, for every $v>0$ the event $B=\{\exists i\ne\ast: g_i(n_i)\ge v\}$ is nondecreasing in $Z_\ast$ (pointwise, for all positive weights, including ties).
\end{lemma}
\begin{proof}
At every step the head of each arm's sequence is its current key, and the greedy rule takes the globally largest head; since each sequence is nonincreasing, this is a $k$-way merge and the consumed set is the top $N-k$ items of the union. Decreasing $Z_\ast$ raises every item of arm $\ast$ and changes no other item, so no other arm's item can move up in the order.
\end{proof}

\begin{lemma}[truncated marginal work]\label{lem:G2}
Let $\bar g_i=\min\{g_i,v\}$ with increments $\bar\delta_i$, let $\lambda>0$, and $L_i=\#\{1\le t\le H:\bar\delta_i(t)/Z_i\ge\lambda\}$. Then $\E[L_i]=\sum_t\min\{1,\bar\delta_i(t)/\lambda\}\le\bar g_i(H)/\lambda\le v/\lambda$.
\end{lemma}
\begin{proof}
$\Pr(Z_i\le\bar\delta_i(t)/\lambda)=\min\{1,\bar\delta_i(t)/\lambda\}$ and the increments telescope.
\end{proof}
The right-hand side is the \emph{total} increment $v$, not ``number of scales times $v$''; this is what removes the logarithm.

\begin{lemma}[generic threshold guarantee]\label{lem:G3}
Let $k\ge2$, $H\ge16k$, curves nonnegative, nondecreasing, discretely concave with $g_i(0)=0$, and a designated arm $\ast$. Suppose an integer $1\le d\le H/4$ and $v>0$ satisfy $g_\ast(d)\ge2v$. Let $r=H-\lfloor H/2\rfloor$, $V_\ast=\sum_{t\le r}g_\ast(t)$ and $p=\min\{1,H/(32kd)\}$. Then $\RMP(H)$ satisfies $\E[\ALG_H]\ge\frac34\min\{rv,\ pV_\ast\}$.
\end{lemma}
\begin{proof}
Let $\lambda=32kv/H$ and $C=\{\sum_{i\ne\ast}L_i\le H/8\}$. By Lemma~\ref{lem:G2} and Markov, $\Pr(C)\ge3/4$; $C$ depends only on $Z_{-\ast}$. Let $A=\{Z_\ast\le p\}$ and $B$ be as in Lemma~\ref{lem:G1}.

\emph{Step 1: on $A\cap C\cap\neg B$ the committed arm is $\ast$.} Suppose $\ast$ has not reached value $v$ by the end of probing; then $n_\ast<d$ and, for all $n\le n_\ast$, concavity gives $\delta_\ast(n)\ge\frac{g_\ast(d)-g_\ast(n)}{d-n}\ge v/d$, so on $A$ the key of $\ast$ is at least $(v/d)/p\ge\lambda$ throughout. Every extra probe given to another arm $i$ therefore used an item with key $\ge\lambda$, and on $\neg B$ (all observed values of the other arms below $v$) these items have $\bar\delta_i=\delta_i$; hence the extra probes on other arms number at most $\sum_{i\ne\ast}L_i\le H/8$. The total number of probes is then at most $k+d+H/8\le\frac H{16}+\frac H4+\frac H8=\frac{7H}{16}<\lfloor H/2\rfloor=N$ for $H\ge32$, a contradiction. So $\ast$ reaches $v$ while every other arm stays below $v$, and $\ast$ is committed.

\emph{Step 2: correlation.} Fix $Z_{-\ast}$ (which fixes $C$). By Lemma~\ref{lem:G1}, $\mathbf 1_{\neg B}$ is nonincreasing in $Z_\ast$, so $\int_0^p\mathbf 1_{\neg B}(z)\,dz\ge p\int_0^1\mathbf 1_{\neg B}(z)\,dz$; integrating over $Z_{-\ast}$ on $C$, $\Pr(A\cap C\cap\neg B)\ge p\Pr(C\cap\neg B)\ge p\,(3/4-\Pr(B))_+$.

\emph{Step 3: rewards.} On $B$ the committed arm has current value $\ge v$, so the commit phase earns $\ge rv$; on $A\cap C\cap\neg B$ it earns $\sum_{j\le r}g_\ast(n_\ast+j)\ge V_\ast$. With $b=\Pr(B)$, $\E[\ALG_H]\ge rv\,b+pV_\ast(3/4-b)_+\ge\frac34\min\{rv,pV_\ast\}$ (the expression is linear in $b$ on $[0,3/4]$).
\end{proof}

\begin{theorem}[prior-free, known horizon]\label{thm:TG}
Let $k\ge2$, $\beta\in(0,1]$, $T\ge16k$, and suppose the optimal arm satisfies $\LE(\beta)$. Then $\RMP(T)$, which reads neither $m$ nor $\beta$, satisfies
\[
\E[\ALG_T]\ \ge\ \frac{3}{256}\,R\,k^{-\beta/(1+\beta)} .
\]
Moreover, for every integer $1\le d\le T/4$ simultaneously,
\[
\frac{\E[\ALG_T]}{R}\ \ge\ \frac{3}{1024}\min\Big\{\frac{f_\ast(d)}{m},\ \frac{T}{kd}\Big\}\qquad\text{(depth-oracle guarantee)}.
\]
\end{theorem}
\begin{proof}
If several arms are cumulatively optimal, fix one that satisfies $\LE(\beta)$ as the comparison arm $i^\circ$ and let $m^\circ=f_{i^\circ}(T)$ play the role of $m$ in the analysis (the algorithm reads neither; the other optimal arms fall into the event $B$); we keep writing $\ast$ and $m$. Let $q=k^{-\gamma}$, $K=kq$, $d=\lceil T/(8K)\rceil\le T/(4K)\le T/4$ (as $T\ge16k\ge8K$) and $v=mq/16$. By $\LE(\beta)$, $f_\ast(d)\ge m(d/T)^\beta\ge m(8K)^{-\beta}=mq8^{-\beta}\ge mq/8=2v$. Also $p\ge K/(8k)=q/8$, $rv\ge(T/2)(mq/16)\ge Rq/32$ and $V_\ast\ge R/8$ by Lemma~\ref{lem:L1}(2). Lemma~\ref{lem:G3} gives $\frac34\min\{Rq/32,Rq/64\}=3Rq/256$. For the depth-oracle bound take $v=f_\ast(d)/2$: $rv\ge Rf_\ast(d)/(4m)$ and $pV_\ast\ge\frac R8\min\{1,T/(32kd)\}$, and $f_\ast(d)/m\le1$.
\end{proof}

\begin{lemma}[conditioning on history]\label{lem:G4}
Suppose a phase of length $H\ge16k$ starts after $t_0$ pulls with $t_0+H\le T$ and arm $i$ pulled $a_i$ times, and runs $\RMP(H)$ with fresh weights on the local curves $g_i(t)=f_i(a_i+t)$, $g_i(0)=0$. These curves are nondecreasing and discretely concave (the first increment $f_i(a_i+2)-f_i(a_i+1)\le f_i(1)\le g_i(1)$; later ones inherit), $g_\ast(t)\ge f_\ast(t)\ge m(t/T)^\beta$, and
\[
\E[\text{phase reward}\mid\text{history}]\ \ge\ \frac{3}{256}\,R\,q_\beta\,(H/T)^{1+\beta}.
\]
\end{lemma}
\begin{proof}
Apply Lemma~\ref{lem:G3} with $d=\lceil H/(8K)\rceil$ and $v=\frac{mq}{16}(H/T)^\beta$; now $rv\ge\frac{Rq}{32}(H/T)^{1+\beta}$ and $V_\ast\ge\frac{m}{T^\beta}\sum_{t\le r}t^\beta\ge\frac{m\,r^{1+\beta}}{(1+\beta)T^\beta}\ge\frac R8(H/T)^{1+\beta}$ using $2^{1+\beta}(1+\beta)\le8$.
\end{proof}

\begin{theorem}[prior-free anytime, $T\ge16k$]\label{thm:TGp}
Run $\RMP(H_j)$ in phases $H_j=16k\cdot2^j$ with fresh weights, without resetting the true pull counts. The algorithm reads only $k$, and for every $\beta\in(0,1]$, every instance whose optimal arm satisfies $\LE(\beta)$, and every $T\ge16k$, $\E[\ALG_T]\ge\frac{3}{4096}Rq_\beta$.
\end{theorem}
\begin{proof}
The last completed phase has length $H_j>T/4$ since $\sum_{\ell\le j+1}H_\ell=4H_j-16k>T$; apply Lemma~\ref{lem:G4} and $4^{-(1+\beta)}\ge1/16$.
\end{proof}

\begin{corollary}[price of priors, noiseless, $T\ge16k$]
Define the price of priors $\varphi$ as the smallest factor by which an algorithm that reads neither $m$ nor $\beta$ falls short of $k^{\beta/(1+\beta)}$, uniformly over $\beta$. Then $\frac13\le\varphi\le128$ with $T$ known and $\frac13\le\varphi\le2048$ with $T$ unknown (the lower bound is the $\beta=1$ class of Corollary~\ref{cor:L4}). Every combination of unknown $m$, $\beta$, $T$ has a $\Theta(1)$ price.
\end{corollary}

\begin{remark}[comparisons]\label{rem:comparisons}
(i) A one-level randomization of $\PC$ over $O(\log k)$ sample sizes ($s_j=\min\{k,\lfloor\sqrt k\rfloor2^j\}$ for $j=0,\dots,J$ with $J=\min\{j:s_j=k\}$ and $h=J+1$, chosen uniformly) gives $\E\ge R/(16h\,k^\gamma)$ with $h\le\lceil\frac12\log_2k\rceil+2$ for $T\ge2k$; this is the same ``guess the level'' technique as \cite[Thm.~11]{BR25} and, on the common horizon range $T\ge16k$, is dominated in order by Theorem~\ref{thm:TG}; we record it only as a comparison. (ii) On the pure truncation families of Lemma~\ref{lem:L4}, a platform-detection rule (visit the arms in a uniformly random order, pulling each until a zero increment is observed; if the order is exhausted, spend the remaining probes on the current best candidate; always keep the arm with the largest observed value and commit to it for the second half; instances with $R_g=0$ are trivial) achieves $\E/R_g\ge\frac1{96}\min\{1,\,Tg(d)/R_g+T/(k(d+1))\}$ for every $(g,d)$ and $T\ge4$, matching Lemma~\ref{lem:L4} up to constants (the case $d+1>\lfloor T/2\rfloor$, including $d=T$, is handled first: the first probed arm already guarantees $Tg(d)/6$ and $Tg(d)/R_g\ge d/T\ge1/3$; only afterwards is $g(d+1)$ read); keeping the best arm is essential (dropping the good arm when it plateaus loses a factor $\Theta(n)$ on $k=T=n^2$, $g=\min(t,n)$, decoys $\min(t,1)$).
\end{remark}

\subsection{All horizons: independent sparse inclusion}

$\RMP$ pulls every arm once and hence needs $H\ge16k$. Restricting it to a random subset of size $\lfloor T/16\rfloor$ conditioned on containing $\ast$ yields only $\E\ge\frac3{256}R\,s^{1/(1+\beta)}/k$, which is not optimal for $\sqrt k<T<k$. The fix is to include arms \emph{independently} and to account for the ``good decoy'' branch without multiplying by the inclusion probability of $\ast$.

\begin{algorithm}[h]
\caption{$\SRMP(k,H)$ (sparse random-marginal probing). Reads only $k$ and $H$.}
\begin{algorithmic}[1]
\State If $k=1$: pull the arm. If $H<32$: pull a uniformly random arm for all $H$ steps.
\State Otherwise let $N=\lfloor H/2\rfloor$, $r=H-N$, $\theta=\min\{1,H/(128k)\}$; draw independent $X_i\sim\mathrm{Bernoulli}(\theta)$, $Z_i\sim\mathrm{Unif}(0,1)$; $S=\{i:X_i=1\}$.
\State If $S=\emptyset$ or $|S|>N$: pull a uniformly random arm for all $H$ steps.
\State Otherwise initialize each $i\in S$ with one pull (counted in $N$); spend the remaining probes greedily on the key $(\delta_i(n_i)/Z_i,1/Z_i)$ over $i\in S$; commit the $r$ remaining pulls to $\argmax_{i\in S}(g_i(n_i),1/Z_i)$.
\end{algorithmic}
\end{algorithm}
For the analysis, in all execution paths (including the fallback) the first $N$ pulls are the ``probe prefix'', and $B$ denotes the event that by the end of this prefix some arm other than the designated one has an observed value $\ge v$; on the fallback path, if $B$ occurs the chosen arm is not the designated one and its true value at pull $N$ is at least $v$, so the last $r$ pulls earn at least $rv$ as well.

\begin{lemma}[sparse hitting lemma]\label{lem:W44}
Let $H\ge32$, let $\ast$ be a designated arm (not necessarily optimal) with $g_\ast(d)\ge2v$ for some $v>0$ and integer $1\le d\le\lfloor H/4\rfloor$. Then $\SRMP$ satisfies
\[
\E[\ALG_H]\ \ge\ \tfrac12\min\{rv,\ \alpha P_\ast(r)\},\qquad \alpha=\min\Big\{\theta,\ \frac{H}{32kd}\Big\}.
\]
\end{lemma}
\begin{proof}
\emph{Counting.} Let $\lambda=32\theta kv/H$, $L_i$ as in Lemma~\ref{lem:G2}, and $C=\{\sum_{i\ne\ast}X_i\le H/32,\ \sum_{i\ne\ast}X_iL_i\le H/8\}$. Since $\E\sum_{i\ne\ast}X_i\le\theta k\le H/128$ and $\E\sum_{i\ne\ast}X_iL_i\le\theta(k-1)v/\lambda\le H/32$ ($X_i$ and $Z_i$ are independent), two Markov inequalities give $\Pr(C)\ge1/2$; $C$ is independent of $(X_\ast,Z_\ast)$. On $C$, $|S|\le H/32+1\le H/16\le N$ even if $\ast$ is included, so the fallback does not occur on $C\cap\{X_\ast=1\}$. Let $p=\min\{1,H/(32\theta kd)\}$ and $A=\{X_\ast=1,\ Z_\ast\le p\}$; $\Pr(A)=\theta p=\alpha$. On $A\cap C\cap\neg B$, if $\ast$ had not reached $v$ then $n_\ast<d$, its key would be at least $(v/d)/p\ge\lambda$ throughout, every extra probe on another arm would use an item with key $\ge\lambda$ and (on $\neg B$) untruncated increment, so the probe count would be at most $|S|+d+\sum_{i\ne\ast}X_iL_i\le\frac H{16}+\frac H4+\frac H8<N$, a contradiction. Hence $\ast$ reaches $v$ and is committed.

\emph{Correlation.} Fix all randomness of the other arms (and the uniform choice used on the fallback path) and condition on $C$; this fixes $C$ and leaves only the pair $(X_\ast,Z_\ast)$ random. Parametrize it by the single variable
\[
Y=\begin{cases}0,&X_\ast=0,\\ 1/Z_\ast,&X_\ast=1,\end{cases}
\]
so that larger $Y$ means ``$\ast$ has higher priority'': $Y=0$ is the lowest atom and, among included states, a smaller $Z_\ast$ gives a larger $Y$. Increasing $Y$ does not increase any other arm's pull count. Indeed, on $C$ the run never takes the fallback branch once $\ast$ is included, because $|S|\le H/32+1\le H/16\le N$ and $S\ne\emptyset$; for two included states this is Lemma~\ref{lem:G1}; and passing from ``not included'' to ``included'' adds one stream to the merge and removes one merge slot, which can only shorten the prefix consumed by the other streams (if the others' subset was empty, the run before inclusion is the fallback on a uniformly chosen arm, and after inclusion it probes only $\ast$ and commits to it, so no other arm is pulled at all). Since $B$ is the event that some \emph{other} arm reaches the level $v$, and other arms' pull counts do not increase, $\mathbf 1_{\neg B}$ is \emph{nondecreasing} in $Y$; $\mathbf1_A=\mathbf1\{Y\ge1/p\}$ is nondecreasing in $Y$ as well. Two bounded nondecreasing functions of one real random variable are nonnegatively correlated, so
\[
\Pr(A\cap\neg B\mid C,Z_{-\ast},X_{-\ast})\ \ge\ \Pr(A)\,\Pr(\neg B\mid C,Z_{-\ast},X_{-\ast})\ =\ \alpha\,\Pr(\neg B\mid\cdots),
\]
and averaging over the other arms' randomness on $C$ gives $\Pr(A\cap C\cap\neg B)\ge\alpha\Pr(C\cap\neg B)\ge\alpha(1/2-\Pr(B))_+$. (The same one-dimensional argument, with $Y=1/Z_\ast$, is what is used in Lemma~\ref{lem:G3}.)

\emph{Branches.} On $B$ the commit phase earns $\ge rv$; on $A\cap C\cap\neg B$ it earns $\ge P_\ast(r)$. Hence $\E\ge rv\,b_0+\alpha P_\ast(r)(1/2-b_0)_+\ge\frac12\min\{rv,\alpha P_\ast(r)\}$.
\end{proof}

\begin{theorem}[prior-free, all horizons]\label{thm:W46}
For all integers $k,T\ge1$, all $\beta\in(0,1]$ and all instances whose optimal arm satisfies $\LE(\beta)$, $\SRMP(k,T)$ satisfies
\[
\E[\ALG_T]\ \ge\ \frac{R}{c_{\mathrm{SR}}}\min\Big\{k^{-\beta/(1+\beta)},\ \frac Tk\Big\},\qquad c_{\mathrm{SR}}:=2^{13}=8192 .
\]
Consequently, the optimal competitive ratio of algorithms reading neither $m$ nor $\beta$ is $\Theta(k^{\beta/(1+\beta)}+k/T)$ for every horizon: $\frac{k^\gamma+k/T}{\beta+3}\le\rho_\beta(k,T)\le\sup_{I\in\mathcal I_\beta,\,R_I>0}\frac{R_I}{\E[\ALG_T(\SRMP,I)]}\le c_{\mathrm{SR}}(k^\gamma+k/T)$ --- one and the same $\SRMP$ attains all of these upper bounds simultaneously --- the lower bound being the known-$\beta$ baseline of Theorem~\ref{thm:TD}(ii).
\end{theorem}
\begin{proof}
For $T\ge32$ take $v=f_\ast(d)/2$ in Lemma~\ref{lem:W44}: $rv\ge\frac R4\frac{f_\ast(d)}m$, $\alpha P_\ast(r)\ge\alpha R/4$ (Lemma~\ref{lem:L1}(4)) and $\alpha\ge\frac1{128}\min\{1,T/(kd)\}$, so $\E/R\ge\frac1{1024}\min\{f_\ast(d)/m,\,T/(kd)\}$ for every $1\le d\le T/4$. Choose $d$ (in the analysis only): if $T\le K_\beta$ take $d=1$, then $T^{-\beta}\ge T/k=Q_\beta$; if $T>K_\beta\ge4$ take $d=\lfloor T/K_\beta\rfloor\in[T/(2K_\beta),T/4]$, then $(d/T)^\beta\ge q_\beta/2$ and $T/(kd)\ge q_\beta$; if $K_\beta<4$ take $d=\lfloor T/4\rfloor\ge T/8$, then $(d/T)^\beta\ge1/8\ge q_\beta/8$ and $T/(kd)\ge4/k\ge q_\beta$. In all cases $\min\{(d/T)^\beta,T/(kd)\}\ge Q_\beta(k,T)/8$, and $\LE(\beta)$ gives $f_\ast(d)/m\ge(d/T)^\beta$. For $T<32$ the uniform single arm earns $\ge R/k\ge RQ_\beta/31$. The cases $k=1$ and $R=0$ are trivial. If several arms are optimal, fix one satisfying $\LE(\beta)$ as $\ast$; the others fall into the good-decoy branch.
\end{proof}

\begin{theorem}[prior-free anytime, all horizons]\label{thm:W47}
Run $\SRMP$ in phases of lengths $1,2,4,\dots$, each with fresh randomness and with the first observation of the phase serving as the first increment. The algorithm reads only $k$, and for every $T\ge1$, $\E[\ALG_T]\ge2^{-19}\,R\,Q_\beta(k,T)$.
\end{theorem}
\begin{proof}
The reset curves $g_i(t)=f_i(n_i+t)$ are nondecreasing and concave, and $g_\ast(d)\ge f_\ast(d)$; Lemma~\ref{lem:W44} applies to the designated arm $\ast$ without requiring it to be optimal for the shifted curves. Let $H$ be the last completed phase, $x=H/T>1/4$. With $v=f_\ast(d)/2$: $rv/R\ge\frac{x^{1+\beta}}4(d/H)^\beta\ge\frac{x^2}4(d/H)^\beta$ and $P_\ast(r)/R\ge x^2/4$ by Lemma~\ref{lem:L1}(2), so the phase earns at least $\frac{Rx^2}{8192}Q_\beta(k,H)\ge\frac{Rx^3}{8192}Q_\beta(k,T)$, and $x^3>1/64$. If $H<32$ then $T\le62$ and the uniform-arm phase earns $\ge\frac1kP_\ast(H)\ge R/(16k)\ge2^{-19}RQ_\beta$.
\end{proof}

\begin{remark}
Two natural alternatives fail by a factor $\Theta(\sqrt k)$ on the instance $T=k/2$, $f_\ast(t)=t$, decoys $\equiv2$: screening arms by their first observed value keeps a decoy (ratio $\frac{k(k+2)}{8k-1}$), and giving uninitialized arms infinite priority spends the whole probe budget on initializations (ratio $\frac{k(k+2)}{8k-2}$). Independent sparse inclusion avoids both. Theorem~\ref{thm:W46} also rules out any Yao-type lower bound showing an unbounded noiseless price of priors: the per-instance guarantee averages over any mixture of $\beta$-classes (using $B_0=1+k/T$, $Q_0=\min\{1,T/k\}$ for instances with envelope exponent $0$, e.g. all arms constant).
\end{remark}

\section{Multiplicative noise: the price of priors jumps}\label{sec:noise}

We work in the noise model of Section~\ref{sec:model}: $\hat f_i(t)\in[af_i(t),bf_i(t)]$, $a=1-\varepsilon$, $b=1+\varepsilon$, fixed by the oblivious adversary; the algorithm observes and earns $\hat f$; the benchmark is $\hat R=\sum_{t\le T}\hat f_\ast(t)\in[aR,bR]$. We never treat $\hat f$ as monotone or concave; all uses of Lemma~\ref{lem:L1} are for the true curves. For an observed prefix let $F_i(n)=\max_{t\le n}\hat f_i(t)$ (running maximum); then $af_i(n)\le F_i(n)\le bf_i(n)$ and, once arm $c$ has been pulled $n_c$ times, every later pull earns $\hat f_c(n_c+j)\ge af_c(n_c+j)\ge af_c(n_c)\ge\frac abF_c(n_c)$. Let $\PhiK(k)=\sqrt{(1+\ln k)/\ln(e+\ln k)}$.

\subsection{Probe-and-commit is robust and needs no $\varepsilon$}

\begin{theorem}[noisy probe-and-commit]\label{thm:H1}
Let $\PC$ commit to the arm with the largest observed probe value. For every instance and admissible $(s,\tau)$,
\[
\E[\widehat{\ALG}_T]\ \ge\ \frac{a^2}{b}\,R\,\min\Big\{\frac{r\tau}{T^2},\frac{s\,r(r+1)}{2kT^2}\Big\}.
\]
The same estimate holds phase-wise, conditionally on the history, exactly as in Lemma~\ref{lem:L5}. Consequently each algorithm of Section~\ref{sec:scale} keeps its guarantee up to the factor $a^2/b$: with $\E[\widehat{\ALG}_T]\ge\frac{a^2}{b}\frac{R}{CG}\ge\frac{a^2}{b^2}\frac{\hat R}{CG}$ one has
\begin{center}\small
\begin{tabular}{lll}
\toprule
algorithm & $CG$ & horizon\\
\midrule
$\PC$, $T$ known (Thm.~\ref{thm:TA}) & $16\sqrt k$ & $T\ge2\lfloor\sqrt k\rfloor$\\
$\PC$, truncated parameters (Thm.~\ref{thm:allT}) & $24(\sqrt k+k/T)$ & all $T\ge1$\\
$\APC$, $T$ unknown (Thm.~\ref{thm:anytimeAll}) & $1536(\sqrt k+k/T)$ & all $T\ge1$\\
doubling $\PC$, $T$ unknown (Thm.~\ref{thm:anytime}) & $256\sqrt k$ & $T\ge2\lfloor\sqrt k\rfloor$\\
$\PC$, $\beta$ known (Thm.~\ref{thm:TD}) & $16/q_\beta$ & $T\ge2\lceil K_\beta\rceil$\\
doubling $\PC$, $\beta$ known, $T$ unknown (Thm.~\ref{thm:TD}(iii)) & $256/q_\beta$ & $T\ge2\lceil K_\beta\rceil$\\
\bottomrule
\end{tabular}
\end{center}
The algorithm reads neither $m$ nor $\varepsilon$; the lower bounds of Section~\ref{sec:scale} persist (zero noise is admissible), so the noisy minimax ratio of scale-oblivious algorithms is $\Theta(\sqrt k+k/T)$ at every horizon, and $\Theta(\sqrt k)$ when $T=\Omega(\sqrt k)$.
\end{theorem}
\begin{proof}
Let $\hat F=\{i\ne\ast:\hat f_i(\tau)\ge\hat f_\ast(\tau)\}$ (fixed once the noise is fixed) and $A,B$ as in Lemma~\ref{lem:L2}. On $B$, $\hat f_c(\tau)\ge\hat f_\ast(\tau)\ge af_\ast(\tau)$ and the commit phase earns $\ge\frac ab r\hat f_c(\tau)\ge\frac{a^2}b rf_\ast(\tau)\ge\frac{a^2}bR\frac{r\tau}{T^2}$. On $A\cap\neg B$ the committed arm is $\ast$ and earns $\ge a\sum_{j\le r}f_\ast(\tau+j)\ge aR\frac{r(r+1)}{2T^2}$. Conclude as in Lemma~\ref{lem:L3} using $a^2/b\le a$. Compare \cite[App.~C]{BR25}, whose algorithm ``must have knowledge of the value of $\varepsilon$'' and, without $m$, is stated there to have approximation factor $O\big((\tfrac{1-\varepsilon}{1+\varepsilon})^2\sqrt k\log(\tfrac{1+\varepsilon}{1-\varepsilon}k)\big)$ (as printed; the first factor is presumably meant as its reciprocal).
\end{proof}

\subsection{Random-marginal probing breaks}

Under noise the observed increments can be negative and non-monotone, so the merge structure of Lemma~\ref{lem:G1} is lost: for $k=2$, $H=32$, both true curves $\equiv1$, $\hat f_1=(1,\tfrac34,\tfrac34,\dots)$, $\hat f_2=(\tfrac54,\tfrac34,\dots)$ and $Z_2=\tfrac12$, decreasing $Z_1$ from $\tfrac34$ to $\tfrac18$ changes the probe counts from $(14,2)$ to $(2,14)$. More seriously, a hidden growth prefix defeats $\RMP$ and two of its natural repairs. We fix the variants precisely, since the phrase admits several readings. The \emph{running-maximum increment} key replaces $\delta_i(n)$ by $\max_{t\le n}\hat f_i(t)-\max_{t\le n-1}\hat f_i(t)$, which guards against non-monotone observations; the \emph{half-window slope} key replaces it by $\big(F_i(n)-F_i(\lfloor n/2\rfloor)\big)/\big(n-\lfloor n/2\rfloor\big)$, where $F_i(0)=0$ and $F_i(n)=\max_{1\le t\le n}\hat f_i(t)$. Both keys vanish for the comparator at its second pull in the example below, so the example defeats both. A third reading, which keeps the \emph{largest} dyadic window --- $\hat f_i(n)-\hat f_i(n-2^{\lfloor\log_2n\rfloor})$, and $\hat f_i(n)/n$ when that index vanishes --- is a different algorithm and is \emph{not} defeated by the example: at $n=2$ the window covers the whole prefix, so the comparator's key there is $\hat f_\ast(2)/2>0$ rather than zero, and on the example below that algorithm commits to the comparator with probability at least $\tfrac14$, so its ratio there is bounded. We claim nothing about it. A variant keyed on the \emph{level} of the running maximum rather than on its increment is a different algorithm: there the comparator's key remains $1/Z_\ast$, comparable to the decoy's, so it is probed further with constant probability and the growth is revealed; the example does not defeat that variant and we claim nothing about it.

\begin{example}[fake growth]\label{ex:N2}
Let $m\ge3$ be an integer, so that $T=H=4m^2\ge16k$ as required by $\RMP$; let $k=2$, $T=H=4m^2$, $\varepsilon=\tfrac14$, $f_\ast(t)=\min\{m,1+\tfrac{t-1}4\}$, $f_2\equiv1$ ($t\ge1$), $\hat f_\ast(1)=\hat f_\ast(2)=1$, $\hat f_\ast(t)=f_\ast(t)$ for $t\ge3$, and $\hat f_2(t)=1+\frac{t-1}{4T}$. All noise lies in the band and is even nondecreasing; the optimal arm satisfies $\LE(1/2)$. The second observation of $\ast$ has zero increment while the decoy's increments stay positive, so after its second probe $\ast$ is never probed again and the decoy is committed (the event that $\ast$ receives only one probe has probability $1/(8T)$ and also ends with the decoy committed). Exactly, $E_m=T+\frac{(T-2)(T-3)}{8T}+\frac{T-2}{32T^2}$ while $\hat R_m=mT-\frac{(m-1)(4(m-1)+1)}2-\frac14$, so $\hat R_m/E_m\sim\frac89m\to\infty$.
\end{example}
The monotone-coupling property itself does not require nonincreasing streams. Consider keys of the form $(\kappa_i(n_i)/Z_i,\,1/Z_i,\,-i)$, where $\kappa_i\ge0$ depends only on the arm's own count, with ties broken by index; then decreasing $Z_\ast$ never increases other arms' counts (delete the moves of $\ast$; the others' moves form a prefix of one fixed greedy sequence). What kills the increment-based algorithms is a single zero key at the wrong moment. Theorem~\ref{thm:H2L}, rather than this example, is what rules out a constant uniform price of adaptation for algorithms reading neither $m$ nor $\beta$, at a fixed positive noise level and asymptotically in $k$; it does not assert that every repair is defeated by Example~\ref{ex:N2}.

\subsection{A matching pair: nested permutation probing and a hidden-growth lower bound}

\begin{lemma}[permutation coupling]\label{lem:HRANK}
Fix nonincreasing integer depths $\ell_1\ge\dots\ge\ell_k\ge1$ with $\sum_j\ell_j\le H/2$. Permute the arms uniformly at random, pull the $j$-th arm $\ell_j$ times, and commit to the arm with the largest running maximum. If the first $s$ positions have depth at least $d$, then (with $H=T$)
\[
\E[\widehat{\ALG}_T]\ \ge\ \frac{a^2}{8b}\,R\,\min\Big\{\frac{f_\ast(d)}m,\ \frac sk\Big\}.
\]
\end{lemma}
\begin{proof}
Take $v=af_\ast(d)$, $A=\{\ast$ among the first $s$ positions$\}$, $B=\{$some other arm ends with running maximum $\ge v\}$. Fix the relative order of the other arms; moving $\ast$ forward pushes the overtaken arms back, which does not increase their depths or running maxima, so $\neg B$ is a prefix in the position of $\ast$ and $\Pr(A\cap\neg B)\ge\frac sk\Pr(\neg B)$. On $B$ the commit earns $\ge\frac ab\frac T2v\ge\frac{a^2}{2b}R\frac{f_\ast(d)}m$; on $A\cap\neg B$ the unique maximizer is $\ast$ and the commit earns $\ge aR/8$.
\end{proof}

\begin{algorithm}[h]
\caption{$\NPA(k,H)$ (nested permutation probing), $k\ge2$, $H\ge16k$; reads neither $m$, $\beta$ nor $\varepsilon$.\label{alg:NPA}}
\begin{algorithmic}[1]
\State Let $h=8\PhiK(k)$ and define the depth envelope $D(x)=\frac1{h\sqrt k}$ for $x\le\frac1{h\sqrt k}$; $D(x)=x^{1/\alpha(x)}$ with $\alpha(x)=\frac{\ln(1/(hx))}{\ln(khx)}$ for $\frac1{h\sqrt k}<x<\frac1h$; $D(x)=0$ for $x\ge\frac1h$.
\State Set $\ell_j=\max\{1,\lceil H\,D((j-1)/k)\rceil\}$; run the procedure of Lemma~\ref{lem:HRANK}. (For $k=1$, pull the arm.)
\end{algorithmic}
\end{algorithm}
The envelope is nonincreasing and $\alpha(x)\in(0,1)$ on the middle piece. The probe budget is controlled by the following lemma, which we isolate because it carries the feasibility of the whole schedule.

\begin{lemma}[depth-budget feasibility]\label{lem:HBUDGET}
Let $k\ge2$, $H\ge16k$ and $h\ge8$, and let $D$ be the depth envelope of Algorithm~\ref{alg:NPA}. Then
\[
k\int_0^1D(x)\,dx\ \le\ \frac1{h^2}+\frac{\ln k}{4h^2\ln h},
\qquad
\sum_{j\le k}\ell_j\ \le\ \Big(k\int_0^1D+D(0)+\frac kH\Big)H .
\]
For $h=8\PhiK(k)$ the first bound is at most $\frac1{32}$, and consequently $\sum_j\ell_j\le\frac{7H}{32}$.
\end{lemma}

\begin{proof}
\emph{The plateau.} On $[0,\frac1{h\sqrt k}]$ the envelope is constant equal to $\frac1{h\sqrt k}$, so this piece contributes $k\cdot\frac1{h\sqrt k}\cdot\frac1{h\sqrt k}=\frac1{h^2}$ to $k\int_0^1D$. This term is exactly the area of the left plateau; it must not be dropped.

\emph{The middle piece.} Writing $p_\beta=q_\beta/h$ one has $D(p_\beta)=p_\beta^{1/\beta}$, and parametrising the middle piece by $\beta\in(0,1]$ turns its contribution into $\ln k\int_0^1h^{-(1+1/\beta)}(1+\beta)^{-2}\,d\beta$. Substituting $y=1/\beta$ (so $d\beta=-dy/y^2$ and $(1+\beta)^{-2}=y^2(1+y)^{-2}$, the sign being absorbed by the reversal of the limits) gives
\[
\ln k\int_1^\infty\frac{h^{-(1+y)}}{(1+y)^2}\,dy\ \le\ \frac{\ln k}4\int_1^\infty h^{-(1+y)}\,dy\ =\ \frac{\ln k}{4h^2\ln h},
\]
using $(1+y)^2\ge4$ for $y\ge1$. Adding the two pieces gives the first display.

\emph{Discretisation.} Since $D$ is nonincreasing, the left-endpoint sum obeys $\sum_{j\le k}D(\frac{j-1}k)\le k\int_0^1D+D(0)$, and $\ell_j=\max\{1,\lceil H\,D(\frac{j-1}k)\rceil\}\le H\,D(\frac{j-1}k)+1$, whence the second display.

\emph{The numerical bound.} For $h=8\PhiK(k)$ we have $h\ge8$ and $\ln h\ge\frac14\ln(e+\ln k)$, so
$\frac1{h^2}\big(1+\frac{\ln k}{4\ln h}\big)\le\frac1{h^2}\big(1+\frac{\ln k}{\ln(e+\ln k)}\big)\le\frac{2}{h^2}\cdot\frac{1+\ln k}{\ln(e+\ln k)}=\frac{2\PhiK(k)^2}{h^2}=\frac1{32}$,
where the middle inequality is $\ln(e+\ln k)\le2+\ln k$. Finally $D(0)=\frac1{h\sqrt k}\le\frac18$ and $\frac kH\le\frac1{16}$, so $\sum_j\ell_j\le(\frac1{32}+\frac18+\frac1{16})H=\frac{7H}{32}$.
\end{proof}

\begin{theorem}[noisy prior-free upper bound]\label{thm:H2U}
For every $\beta\in(0,1]$, every instance whose optimal arm satisfies $\LE(\beta)$, and $T\ge16k$,
\[
\E[\widehat{\ALG}_T]\ \ge\ \frac{a^2}{64b}\,\frac{R\,q_\beta}{\PhiK(k)}\ \ge\ \frac{a^2}{64b^2}\,\frac{\hat R\,q_\beta}{\PhiK(k)} .
\]
With $T$ unknown (phases $16k\cdot2^j$, conditioning on history as in Lemma~\ref{lem:G4}) the constant becomes $1024$.
\end{theorem}
\begin{proof}
Let $s=\lceil kq_\beta/h\rceil$ and $d=\lceil T(q_\beta/h)^{1/\beta}\rceil$. For $j\le s$, $(j-1)/k<q_\beta/h$, so $\ell_j\ge\lceil T\,D(q_\beta/h)\rceil=d$; also $s/k\ge q_\beta/h$ and $f_\ast(d)/m\ge(d/T)^\beta\ge q_\beta/h$. Apply Lemma~\ref{lem:HRANK} with $h=8\PhiK(k)$.
\end{proof}

\begin{theorem}[noisy prior-free lower bound; asymptotic]\label{thm:H2L}
Fix $\varepsilon\in(0,\tfrac12]$ and let
\[
r_\varepsilon=\max\{2,\lceil\log_2(1/\varepsilon)\rceil\},\qquad k_0(\varepsilon)=2^{\max\{256,\,3r_\varepsilon\}}.
\]
For every $k\ge k_0(\varepsilon)$ and every algorithm that does not read $m$ or $\beta$ (it may read $T$ and $\varepsilon$), there exist $\beta\in(0,1]$, an instance whose optimal arm satisfies $\LE(\beta)$, and a fixed in-band noise table such that
\[
\E[\widehat{\ALG}_T]\ \le\ \frac{8192}{\varepsilon\,\PhiK(k)}\,\hat R\,q_\beta .
\]
The bound is proved at the horizon $T=16k'$ (for $k=k'$ a power of two) or $T=32k'$ (general $k$, $k'=2^{\lfloor\log_2k\rfloor}$, remaining arms identically zero); it is an asymptotic statement in $k$. Together with Theorem~\ref{thm:H2U}, the uniform price of not knowing $m$ and $\beta$ under fixed positive noise is $\phi_\varepsilon(k)=\Theta_\varepsilon(\PhiK(k))=\Theta_\varepsilon\big(\sqrt{\log k/\log\log k}\big)$, with $\phi_\varepsilon$ exactly as defined in Section~\ref{sec:model}: the supremum runs over all $T\ge16k$ and the algorithm is chosen before $\beta$. The lower bound is established only at the stated hard horizons, which suffices for that supremum; the theorem does not assert the same lower bound at every fixed $T\ge16k$; the quantifier is ``for every algorithm there exist $\beta$ and an instance'', and the $\beta=1$ guarantee of Theorem~\ref{thm:H1} is not contradicted. Since $k_0(\varepsilon)\ge2^{256}$, the statement is structural: it separates a bounded price (Theorem~\ref{thm:W46}, noiseless) from one that is unbounded in $k$; at the threshold the explicit constant is tiny (for $\varepsilon=\tfrac12$, $\varepsilon\PhiK(2^{256})/8192\approx3.6\cdot10^{-4}$), so no numerically significant loss is claimed at practical values of $k$.
\end{theorem}
\begin{proof}
\emph{Construction.} Let $k=2^L$, $L\ge\max\{256,3r_\varepsilon\}$, $\eta=2^{-r_\varepsilon}\in[\varepsilon/2,\varepsilon]$, $T=16k$, $B=\lceil\frac12\log_2\frac L{\log_2L}\rceil$, $A=2^B$. Types $j\in\mathcal J=\{j_{\min},j_{\min}+2,\dots\le\lfloor L/2\rfloor\}$, $j_{\min}=\lceil\frac L2-\frac L{16B}\rceil$, with $\beta_j=\frac j{L-j}$, $q_j=2^{-j}$, $m_j=A2^j$, $D_j=2^{\lfloor\log_2T-(L-j)-B(L-j)/j\rfloor}$, $d_j=\eta D_j$ (positive integers). The star arm $U$ is uniform in $[k]$; all other arms have $f=\hat f\equiv1$ for $t\ge1$; the star arm has $f_U(t)=\min\{m_j,1+\frac{t-1}{D_j}\}$ and $\hat f_U(t)=1$ for $t\le d_j$ (hidden prefix; in band since $f_U\le1+\eta\le\frac1{1-\varepsilon}$ there), $\hat f_U(t)=f_U(t)$ for $t>d_j$. The star arm is nondecreasing and discretely concave, satisfies $\LE(\beta_j)$ (as $m_j\le(T/D_j)^{\beta_j}$ and $1+\frac{t-1}D\ge\max\{1,t/D\}\ge(t/D)^\beta$), reaches $m_j$ before $T$, and $\hat R_j\ge aR_j\ge m_jT/4$; the constant decoys satisfy every $\LE(\beta)$. For $\varepsilon=0$ the hidden prefix is illegal: the second pull of the star arm reveals a positive increment.

\emph{Common history.} Fix the algorithm's seed and run it against the all-ones response; let $N(d)$ be the number of arms pulled at least $d+1$ times. The real run coincides with this one until the star arm is pulled for the $(d_j+1)$-st time, so type $j$ is revealed with probability exactly $N(d_j)/k$. Unrevealed, every pull earns at most $1$; revealed, the total is at most $\hat R_j$ since $\hat f_U\ge1$ is nondecreasing. Hence $\E[\widehat{\ALG}_j]\le T+\frac{\E N(d_j)}k\hat R_j$.

\emph{Multi-scale budget and Yao mixture~\cite{Yao77}.} Consecutive types satisfy $d_{j+2}\ge4d_j$, so $\sum_jd_jN(d_j)\le\frac43T\le2T$. Let $S=\sum_jd_jkq_j$ and weights $w_j=d_jkq_j/S$. Then $\sum_jw_j\frac{\E[\widehat{\ALG}_j]}{\hat R_jq_j}\le\frac4A+\frac{2T}S$. The finite-parameter checks behind this display are the following, and we record them in full since they carry the construction. From $L\ge256$ we get $B\ge3$, since $L/\log_2L$ is increasing and already $32$ at $L=256$. For the upper bound, $\lceil x\rceil\le x+1$ gives $2B\le\log_2\frac L{\log_2L}+2=\log_2L-\log_2\log_2L+2$, so that
\[
\log_2L-2B\ \ge\ \log_2\log_2L-2\ \ge\ 1\qquad(L\ge256),
\]
and in particular $B\le\frac12\log_2L$. The analytic lower bound $\log_2\log_2L-2$ is increasing for $L\ge256$, so the displayed bound holds on the full, unbounded parameter range; no finite numerical check is used. Every admissible type index satisfies $j\ge\frac L2-\frac L{16B}\ge\frac{23L}{48}$, the last step using $B\ge3$. Next $\frac1{\beta_j}\le\frac{8B+1}{8B-1}=1+\frac2{8B-1}\le1+\frac2{7B}$, valid for $B\ge1$. With $e_j=\log_2D_j=\lfloor j+4-B/\beta_j\rfloor$ this gives $e_j\ge\frac{23L}{48}+3-B-\frac27\ge\frac L3$, and since $r_\varepsilon\le\frac L3$ the depth $d_j=2^{e_j-r_\varepsilon}$ is a positive integer. For consecutive admissible types the unrounded exponent gap is $2+\frac{2BL}{j(j+2)}\ge2$, and $\lfloor x+d\rfloor-\lfloor x\rfloor\ge\lfloor d\rfloor$ for $d\ge2$, so $e_{j+2}-e_j\ge2$, i.e.\ $d_{j+2}\ge4d_j$ as used above. Finally $e_j\le j+4-B/\beta_j$ gives $B+j+e_j\le4+2j+B(1-\frac1{\beta_j})\le L+4=\log_2T$, i.e.\ $m_jD_j\le T$, so the star arm reaches its plateau inside the horizon; the bound on the number of types follows from the length of the admissible index range minus the two rounding losses at its ends. The remaining finite checks are equally short. Since $A=2^B$ with $B=\lceil\frac12\log_2\frac L{\log_2L}\rceil$ we have $A\in[\sqrt{L/\log_2L},2\sqrt{L/\log_2L}]$. For the number of types put $u=\lfloor L/2\rfloor$ and $\ell=\lceil L/2-L/(16B)\rceil$, so that $|\mathcal J|=\lfloor\frac{u-\ell}2\rfloor+1\ge\frac L{32B}-1$; the two rounding losses at the ends of the range are what the $-1$ absorbs. Since $L\ge256$ and $B\le\frac12\log_2L$ give $64B\le32\log_2L\le L$, we may drop the $-1$ at the cost of a factor $2$, obtaining $|\mathcal J|\ge\frac L{64B}\ge\frac L{32\log_2L}$, which is the bound used below. Next $A^{1/\beta_j}=2^{B/\beta_j}\le2^{B+2/7}=A\cdot2^{2/7}\le2A$ by the bound on $1/\beta_j$ above. Finally $d_jkq_j=\eta2^{e_j+L-j}\ge\eta2^{L+3-B-2/7}=\frac{\eta T}{2\cdot2^{2/7}A}\ge\frac{\eta T}{4A}$. With these, we get $d_jkq_j\ge\frac{\eta T}{4A}$ and $S\ge\frac{\eta|\mathcal J|T}{4A}$, and so the weighted average is at most $\frac4A+\frac{8A}{\eta|\mathcal J|}$; here $\frac4A\le4\sqrt{\log_2L/L}$ and $\frac{8A}{\eta|\mathcal J|}\le\frac{16\sqrt{L/\log_2L}\cdot32\log_2L}{\eta L}=\frac{512}\eta\sqrt{\log_2L/L}$, so the sum is at most $\frac{516}\eta\sqrt{\log_2L/L}$ using $\eta\le1$. Some type $j$ (and then some fixed position of the star arm) attains this bound; $\eta\ge\varepsilon/2$ and $\PhiK(2^L)\le2\sqrt{L/\log_2L}$ give the constant $2064/\varepsilon$ for powers of two. For general $k$ pad with zero arms, take $T=32k'$, and lose at most $\sqrt2\cdot\sqrt2$ in $q_\beta$ and $\PhiK$; $8192$ suffices.
\end{proof}

\subsection{Knowing either prior alone restores a constant price}

\begin{algorithm}[h]
\caption{$\RET(H,\mu)$ (random-exponent elimination); reads $k$, $H$ and a scale $\mu$; not $\beta$ or $\varepsilon$.}
\begin{algorithmic}[1]
\State Draw $U_i\sim\mathrm{Unif}(0,1)$ and test the arms in decreasing order of $U_i$. Let $\alpha(u)=1$ if $u\le k^{-1/2}$ and $\alpha(u)=\frac{\ln(1/u)}{\ln k-\ln(1/u)}$ otherwise.
\State Keep pulling the current arm until the first $n$ with $\hat g_i(n)<\frac\mu2(n/H)^{\alpha(U_i)}$; then eliminate it and move on. Probe budget $\lfloor H/2\rfloor$; when it is exhausted, commit the rest of the phase to the arm with the largest running maximum. If all arms are eliminated before the probe budget is exhausted, \emph{immediately} commit all remaining pulls of the phase to the arm with the largest running maximum observed so far.
\end{algorithmic}
\end{algorithm}

\begin{theorem}[known scale]\label{thm:RET}
Let $k\ge2$, $\varepsilon\in[0,\tfrac12]$, $H\ge16k$ and $\beta\in(0,1]$. Suppose the comparator arm satisfies $g_\ast(t)\ge\mu(t/H)^\beta$ for all $1\le t\le H$, where $\mu>0$ is the number supplied to the algorithm, and that the remaining true curves are nonnegative and nondecreasing. Then
\[
\E[\text{noisy reward of }\RET(H,\mu)]\ \ge\ \frac a{2048\,b}\,\mu Hq_\beta .
\]
In particular, when the scale of a cumulatively optimal arm is known --- take $H=T$ and $\mu=m=f_\ast(T)$, so that $\LE(\beta)$ for that arm supplies the hypothesis --- we get $\E[\widehat{\ALG}_T]\ge\frac a{2048b}Rq_\beta$ simultaneously for every such $\beta$, without reading $\beta$ or $\varepsilon$ (using $R\le mT$). If in addition $T$ is unknown, running $\RET(H_j,\mu)$ in phases $H_j=16k\cdot2^j$ with the advice $\mu=m/4$ gives $\E[\widehat{\ALG}_T]\ge\frac a{32768\,b}Rq_\beta$ for every $T\ge16k$, in the advice model of Section~\ref{sec:model}.
\end{theorem}

The proof is given in full in Appendix~\ref{app:ret}; it is the only place where the convention $\varepsilon\le\tfrac12$ is used. Its shape is: with $q=q_\beta$, $p=q/2$, $c=1/512$ and $v=c\mu q$, the comparator is never eliminated once $U_\ast\le q$; the expected work spent on the other arms before one of them reaches the level $v$ is controlled by an integral estimate for the exponent function $\alpha(\cdot)$; and two disjoint events --- some other arm reaches $v$, or the comparator is tested in time and is the only arm at that level --- give the two branches.

\begin{proposition}[small noise: a blocked confidence-key algorithm]\label{prop:W48}
For $\bar\varepsilon\in[0,\tfrac12]$ let $c_{\bar\varepsilon}=\frac{1+\bar\varepsilon}{1-\bar\varepsilon}-\frac{1-\bar\varepsilon}{1+\bar\varepsilon}$ and $D(k,T,\bar\varepsilon)=1+c_{\bar\varepsilon}\min\{T,32k\}$. There is an algorithm $\SRMPc$ (Appendix~\ref{app:blocked}) which reads $k$, $T$ and any valid upper bound $\bar\varepsilon\ge\varepsilon$ on the noise level --- but neither $m$, nor $\beta$, nor $\varepsilon$ itself --- and which satisfies, for all integers $k,T\ge1$, every $\beta\in(0,1]$ and every instance whose comparator satisfies $\LE(\beta)$,
\[
\E[\widehat{\ALG}_T]\ \ge\ \frac{a^2}{65536\,b\,D(k,T,\bar\varepsilon)}\,R\,Q_\beta(k,T).
\]
Consequently, for every fixed $k$,
\[
\limsup_{\varepsilon\downarrow0}\ \sup_{T\ge1}\ \sup_{\beta,I,\hat f}\ \frac{R\,Q_\beta(k,T)}{\E[\widehat{\ALG}_T]}\ \le\ 65536 ,
\]
uniformly in $T$: for a fixed number of arms a vanishing noise level restores a constant price at every horizon. This is a different limit from the fixed-$\varepsilon$, $k\to\infty$ regime of Theorem~\ref{thm:H2L}; the two regimes do not interact.
\end{proposition}
Bounds in Proposition~\ref{prop:W48} are stated against the uncorrupted benchmark $R$; against $\hat R$ multiply the ratio by $b$.

\begin{remark}[what the counterexample does and does not show]
Example~\ref{ex:N2} refutes the raw-increment keys, the running-maximum \emph{increment} keys and the half-window slope keys fixed above; as noted there, a key based on the \emph{level} of the running maximum, and the reading that keeps the largest dyadic window, are different algorithms that the example does not refute. For some weight vectors the \emph{unblocked} confidence key also commits to the decoy there; this is not a competitive-ratio counterexample, since its probability of committing to the comparator is exactly $4T/(8T+5)\to1/2$. Blocking is what the proof of Proposition~\ref{prop:W48} needs in order to replace the factor $1+c_\varepsilon T$ by $1+O(\varepsilon k)$, that is, to obtain a bound uniform in $T$; we do not claim that blocking is necessary.
\end{remark}

\begin{corollary}[one algorithm that reads nothing]\label{cor:mixture}
Let $A$ be the algorithm that, before the run, selects uniformly at random one of the three anytime algorithms $\APC$ (Theorem~\ref{thm:anytimeAll}), the phased $\NPA$ of Theorem~\ref{thm:H2U} and the anytime $\SRMP$ of Theorem~\ref{thm:W47}, and then runs the selected one to the end. It reads none of $m,\beta,T,\varepsilon$, and
\[
\E[\widehat{\ALG}_T]\ \ge\ \frac{a^2}{b}\cdot\frac{R}{4608(\sqrt k+k/T)}\quad(\text{all }T\ge1),\qquad
\E[\widehat{\ALG}_T]\ \ge\ \frac{a^2}{3072\,b}\cdot\frac{R\,q_\beta}{\PhiK(k)}\quad(T\ge16k),
\]
and, when $\varepsilon=0$, also $\E[\ALG_T]\ge R\,Q_\beta(k,T)/(3\cdot2^{19})$ for \emph{every} $T\ge1$. Replacing the third branch by the phased $\RMP$ of Theorem~\ref{thm:TGp} improves the noiseless constant to $R\,q_\beta/4096$ but restricts that guarantee to $T\ge16k$; the two mixtures are different algorithms and one must choose.
\end{corollary}
\begin{proof}
Each branch is selected with probability $1/3$ and all rewards are nonnegative, so each guarantee of a branch survives division by $3$: $1536\cdot3=4608$ for the first line, $1024\cdot3=3072$ for the second, and $2^{-19}\cdot\frac13$ for the third. The third line uses Theorem~\ref{thm:W47}, which is proved in the noiseless model only; under noise that branch contributes only through the nonnegativity of its rewards. Since Theorem~\ref{thm:W47} holds for every $T\ge1$ and is stated with $Q_\beta(k,T)=\min\{q_\beta,T/k\}$, the third line inherits both properties; the variant using Theorem~\ref{thm:TGp} instead would give $\frac3{4096}\cdot\frac13=\frac1{4096}$ but only for $T\ge16k$ and with $q_\beta$ in place of $Q_\beta$.
\end{proof}

\begin{table}[h]
\centering\small
\begin{tabular}{lcc}
\toprule
Unknown among $\{m,\beta\}$ & noiseless price (Section~\ref{sec:priorfree}) & price under fixed $\varepsilon\in(0,\tfrac12]$\\
\midrule
none / only $m$ / only $\beta$ & $\Theta(1)$ & $\Theta(1)$ (Thms.~\ref{thm:H1}, \ref{thm:RET})\\
both $m$ and $\beta$ & $\Theta(1)$ (Thms.~\ref{thm:TG}, \ref{thm:W46}) & $\Theta_\varepsilon(\PhiK(k))$ (Thms.~\ref{thm:H2U}, \ref{thm:H2L})\\
\bottomrule
\end{tabular}
\caption{The uniform price of adaptation $\phi_\varepsilon(k)$ of Section~\ref{sec:model}, i.e.\ $\inf$ over algorithms (or anytime policies) of the worst case over $T\ge16k$, over $\beta$, over instances of $\mathcal I_\beta$ and over admissible noise tables, of $\hat R\,q_\beta/\E[\widehat{\ALG}_T]$. The algorithm is chosen before $\beta$ and the instance. The noisy lower bound is established at the hard horizons $T=16k'$, $32k'$ of Theorem~\ref{thm:H2L} and is asymptotic in $k$; it is not claimed at every fixed $T\ge16k$. Whether $T$ or $\varepsilon$ may be read does not change the order in any cell: the upper bounds do not read them, the lower bound allows them. Each cell is a separate infimum: $\phi_\varepsilon$ is defined relative to a class $\mathfrak A$ fixed by what the algorithms may read, so the entries of different rows are values of the same functional taken over different $\mathfrak A$. In the known-scale cells the scale advice refers to the same designated optimal arm as the envelope hypothesis and the noisy benchmark; the arm's index itself is not supplied to the algorithm. If instead the advice is allowed to come from an arbitrary cumulatively optimal arm, with terminal value $m_0\le2m^\circ$ by Lemma~\ref{lem:L1}(5), running $\RET(T,m_0/2)$ for known $T$ and phases with advice $m_0/8$ for unknown $T$ gives the constants $4096$ and $65536$ respectively, so the $\Theta(1)$ entries are unchanged.}
\label{tab:price}
\end{table}

\section{Deterministic algorithms, positioning, and open problems}\label{sec:discussion}

\paragraph{Deterministic algorithms.} Probing every arm $\lfloor T/(2k)\rfloor$ times and committing to the best probe value is deterministic, reads neither $m$ nor $\beta$, and earns at least $R/(8k)$ for $T\ge2k$ (by Lemma~\ref{lem:L1}(1): the committed arm has $f_c(\tau)\ge f_\ast(\tau)\ge m\tau/T$); conversely, against any deterministic algorithm, replacing an arm that the virtual all-decoy run pulls at most $\lfloor T/k\rfloor$ times by a linear arm forces ratio $\ge k/2$. The upper bound above is proved here independently for a known horizon; the deterministic lower-bound construction follows~\cite{Pat23}. We make no novelty claim for these deterministic bounds. Some horizon restriction is unavoidable, though we do not claim that $2k$ is the sharp boundary: at $T=2k-1$, probing every arm once and committing to the best probe value already gives ratio at most $2k-1$, so the order is still $\Theta(k)$ there. For $T<k$ a deterministic algorithm leaves some arm unpulled on the all-zero trajectory; making that arm linear and the others identically zero gives $\ALG_T=0$ and an unbounded ratio. Even at $T=k$ the ratio is $\Theta(k^2)$: either some arm is never visited, and the ratio is again unbounded, or every arm is pulled exactly once, in which case setting the arm visited last to $f_\ast(t)=t$ and the others to zero leaves the run unchanged until the final step, so $\ALG_T=1$ while $R=k(k+1)/2$. Conversely, deterministically visiting every arm once gives $\ALG_k=\sum_if_i(1)\ge f_\ast(1)$, and discrete concavity with $f_\ast(0)=0$ gives $f_\ast(t)\le tf_\ast(1)$, whence $R\le\frac{k(k+1)}2f_\ast(1)\le\frac{k(k+1)}2\ALG_k$; the deterministic minimax ratio at $T=k$ is therefore exactly $k(k+1)/2$. We therefore state the $\Theta(k)$ characterization on the range $T\ge2k$, which is sufficient for every use made of it here, and make no claim about the exact horizon at which the order changes.

\paragraph{Why relative comparisons remove the scale dependence.} The random round-robin of~\cite{BR25} and the $\mathrm{PTRR}$ family of~\cite{BGRS26} decide whether to keep an arm by an \emph{absolute} test $f_i(t)\ge m(t/T)^\alpha$; the test needs $m$, and the route taken in~\cite{BR25}---estimating $m$ from $T/(2k)$-pull explorations, which localizes it within a factor $4k$, then guessing a dyadic level---costs the $\log k$. We describe the technical difference only. Probe-and-commit replaces the absolute test by a relative one ($\argmax$ over probe values), which is invariant to rescaling, and buys the missing information with a $\sqrt k$ subsample; the $1/\sqrt k$ probability of sampling the optimal arm is precisely what the $\Omega(\sqrt k)$ lower bound already charges. The random-marginal and sparse variants make the ``keep or leave'' decision relative as well (observed increments weighted by fixed random priorities), and Lemma~\ref{lem:G2} shows that the work wasted on decoys is paid by their total value, with no union bound over scales. The same relative comparisons are what make Theorem~\ref{thm:H1} robust to noise \emph{without} knowing $\varepsilon$, in contrast to \cite[App.~C]{BR25}, where the threshold is deflated by $(1-\varepsilon)$ and ``the algorithm therefore must have knowledge of the value of $\varepsilon$''.

\paragraph{Positioning.} \cite{BGRS26} obtain instance-adaptive guarantees through a parameter learned from offline instances and through a data-driven regret/competitive-ratio hybrid; in the introduction of their Appendix~D they write about the hybrid ``We note that as a result of this approach, we do not get per-instance worst-case guarantees. This is an interesting direction for future work,'' and in Appendix~D.2, in the same context, ``that algorithm may not provide the worst-case fallback guarantee on a fixed instance.'' These remarks concern their hybrid algorithm rather than the unknown-scale question. Theorems~\ref{thm:TG}, \ref{thm:W46} and \ref{thm:RET} provide per-instance worst-case guarantees that adapt to $\beta$ without reading it, addressing one instance of the direction they describe, for the noiseless model and for the noisy model with known scale; Theorem~\ref{thm:H2L} shows that in the noisy model the per-instance adaptation to an unknown exponent $\beta$ is impossible without a growing price once the scale is also unknown. The stochastic rising-bandit literature (e.g.\ Metelli et al.~\cite{Met22}) studies instance-dependent regret with sub-Gaussian noise; our objective and adversary are different, and we do not compare constants.

\paragraph{Open problems.}
\begin{enumerate}[leftmargin=*]
\item The optimal dependence on $\varepsilon$ in Theorem~\ref{thm:H2L}: the lower bound scales as $\varepsilon\PhiK(k)$ while the corresponding upper bound on $\phi_\varepsilon(k)$ from Theorem~\ref{thm:H2U} is $(1-\varepsilon)^{-2}(1+\varepsilon)^2\cdot64\,\PhiK(k)$ (the extra factor $b$ converts the guarantee, which is stated against $R$, into the $\hat R$-based price of Section~\ref{sec:model}); we do not know whether the linear factor $\varepsilon$ can be replaced by a constant (the hidden-prefix length in our construction is proportional to $\varepsilon$).
\item Short horizons under noise: the prior-free guarantee of Theorem~\ref{thm:H2U} is for $T\ge16k$; for $32\le T<16k$ the subset bounds involve $\PhiK(s_T)$ with $s_T=\min\{k,\lfloor T/16\rfloor\}$, and this paper does not establish a matching lower bound for that range; smaller horizons are handled separately by the uniform-arm or truncated-$\PC$ branches.
\item Constants: the gap between $4\sqrt3\sqrt k$ and $\sqrt k/3$, and between $3/256$ and the $\beta$-dependent lower bounds.
\item Stochastic noise on the rewards. Here the multiplicative model must be replaced with care: with additive sub-Gaussian noise of a fixed variance, no competitive guarantee of the present form can survive unconditionally. Concretely, let the observations carry independent additive Gaussian noise $N(0,\sigma^2)$ with a fixed $\sigma>0$, compare against the uncorrupted benchmark $R=\delta T$, and place the single arm of constant mean $\delta>0$ at a uniformly random position $U\in[k]$, all other arms having mean zero. (The position must be randomised: for a fixed designated position the algorithm that always pulls that arm has ratio $1$. Nor does ``sub-Gaussian'' alone suffice, since zero noise is sub-Gaussian and some discrete noises make an arbitrarily small shift exactly identifiable.) Write $P_0$ for the law of the history under all-zero means and $P_u^\delta$ for the law when arm $u$ has mean $\delta$, and let $N_u$ be the number of pulls of arm $u$. The chain rule for Kullback--Leibler divergence gives $\mathrm{KL}(P_0\Vert P_u^\delta)=\frac{\delta^2}{2\sigma^2}\E_0N_u$, so by Pinsker's inequality $\E_u^\delta N_u\le\E_0N_u+\frac{T\delta}{2\sigma}\sqrt{\E_0N_u}$. Averaging over $u$ and using $\sum_u\E_0N_u=T$ with Cauchy--Schwarz, $\frac1k\sum_u\E_u^\delta N_u\le\frac Tk+\frac{T\delta}{2\sigma}\sqrt{T/k}$. Hence for every algorithm some position forces ratio at least $k-o(1)$ as $\delta\downarrow0$, and letting $k$ grow rules out any uniform $o(k)$ multiplicative guarantee without further signal assumptions. What is \emph{not} ruled out is a guarantee of order $k$: pulling one uniformly random arm throughout always achieves that. So the casualty is the scale-free $O(k^{\gamma})$-type guarantee of this paper, not competitiveness as such. The question we would like answered is therefore: under which signal-to-noise, minimum-scale or horizon conditions --- or against which benchmark corrected for an unavoidable estimation loss --- does a prior-free competitive guarantee exist for stochastic noise? Probe-and-commit with confidence-adjusted probe values is the natural candidate on the positive side; the increment-based prior-free algorithms would need new ideas in any case.
\end{enumerate}

\appendix

\section{Proof of Theorem~\ref{thm:RET} (known scale, unknown \texorpdfstring{$\beta$}{beta} and \texorpdfstring{$\varepsilon$}{eps})}\label{app:ret}

Throughout, $k\ge2$, $\varepsilon\in[0,\tfrac12]$ (so $a=1-\varepsilon\ge\tfrac12$ and $b=1+\varepsilon\le\tfrac32$), $H\ge16k$, $\beta\in(0,1]$, and the comparator satisfies $g_\ast(t)\ge\mu(t/H)^\beta$ for $1\le t\le H$. Write $K=\ln k$,
\[
q=q_\beta=k^{-\gamma},\qquad p=\frac q2,\qquad c=\frac1{512},\qquad v=c\mu q,\qquad w=\frac{2v}{\mu}=2cq=\frac q{256}.
\]
Since $k\ge2$ and $\beta>0$ we have $q<1$, hence $0<w\le p\le1$. Recall $\alpha(u)=1$ for $u\le k^{-1/2}$ and $\alpha(u)=\frac{\ln(1/u)}{K-\ln(1/u)}$ for $u>k^{-1/2}$; in the substitution $x=\ln(1/u)$ the second branch is $\alpha=x/(K-x)$, which increases from $0$ to $1$ as $x$ goes from $0$ to $K/2$, so $\alpha$ is continuous at $u=k^{-1/2}$, takes values in $[0,1]$, and is nonincreasing in $u$. The value $\alpha(u)=0$ occurs only at $u=1$; there the elimination threshold is the constant $\mu/2>v$, so the truncated test of Lemma~\ref{lem:ret-len} ends after a single pull, and we use throughout the convention $w^{1/\alpha}:=0$ for $\alpha=0$ and $0<w<1$, which is consistent with that lemma. The set $\{u=1\}$ is null and does not affect the integral estimates.

\begin{lemma}\label{lem:ret-alpha}
If $u\le q$ then $\alpha(u)\ge\beta$.
\end{lemma}
\begin{proof}
$x=\ln(1/u)\ge\ln(1/q)=\gamma K$. If $x\ge K/2$ then $\alpha(u)=1\ge\beta$. Otherwise $\alpha(u)=x/(K-x)$ is increasing in $x$, so $\alpha(u)\ge\gamma K/(K-\gamma K)=\gamma/(1-\gamma)=\beta$.
\end{proof}

\begin{lemma}[the comparator is never eliminated]\label{lem:ret-safe}
If $U_\ast\le q$ then the elimination test never fires for the comparator: $\hat g_\ast(n)\ge\frac\mu2(n/H)^{\alpha(U_\ast)}$ for every $1\le n\le H$.
\end{lemma}
\begin{proof}
$\hat g_\ast(n)\ge a\,g_\ast(n)\ge\frac12\mu(n/H)^\beta\ge\frac\mu2(n/H)^{\alpha(U_\ast)}$, using $a\ge\frac12$, the envelope hypothesis, $n/H\le1$ and Lemma~\ref{lem:ret-alpha}. The test eliminates only on a strict inequality.
\end{proof}

For an arm $i$ and a weight $u\in(0,1]$ let $\ell_i(u)$ be the \emph{truncated test length}: the number of pulls that arm $i$ would receive if it were tested with weight $u$, until it is eliminated or its observed value first reaches $v$, whichever happens first (and $H$ if neither happens). This depends only on the curve of arm $i$, on the (fixed) noise table and on $u$; it does not depend on the run.

\begin{lemma}[truncated test length]\label{lem:ret-len}
$\ell_i(u)\le1+H\,w^{1/\alpha(u)}\le1+Hw$ for every $i$ and $u$.
\end{lemma}
\begin{proof}
For $\alpha(u)>0$ put $n_0=\lceil Hw^{1/\alpha(u)}\rceil$. At the $n_0$-th pull the threshold $\frac\mu2(n_0/H)^{\alpha(u)}$ is at least $\frac{\mu w}2=v$: indeed $n_0\ge Hw^{1/\alpha(u)}$ gives $(n_0/H)^{\alpha(u)}\ge w$. Hence at that pull either the observation is at least the threshold, and therefore at least $v$, so the truncated test has ended; or it is strictly below the threshold and the arm is eliminated, which also ends it. Therefore $\ell_i(u)\le n_0\le1+Hw^{1/\alpha(u)}\le1+Hw$, the last step because $w\le1$ and $1/\alpha(u)\ge1$. For $\alpha(u)=0$ (only $u=1$) the test ends after one pull. If a pull both reaches $v$ and triggers elimination, we record it as reaching $v$, consistently with the algorithm's use of the running maximum.
\end{proof}

\begin{lemma}[integral estimate]\label{lem:ret-int}
For $k\ge2$ and $0<w\le p\le1$,\quad $k\int_p^1w^{1/\alpha(u)}\,du\le 2w/p$.
\end{lemma}
\begin{proof}
Write $L=\ln(1/p)$ and $W=\ln(1/w)$, so $W\ge L\ge0$. Substituting $u=e^{-x}$, the integral becomes $\int_0^{L}k\,w^{1/\alpha(e^{-x})}e^{-x}\,dx$. On the uncapped range $x<K/2$ we have $1/\alpha=K/x-1$, so the integrand equals $e^{E(x)}$ with
\[
E(x)=K+W-\frac{WK}{x}-x,\qquad E'(x)=\frac{WK}{x^2}-1 .
\]

\emph{Case 1: $L\le K/2$.} The whole range is uncapped. For $0<x\le L$, $x^2\le L^2\le LW\le\frac K2W$, so $E'(x)\ge1$ and
\[
\int_0^Le^{E(x)}dx\le\int_0^LE'(x)e^{E(x)}dx=e^{E(L)}-\lim_{x\downarrow0}e^{E(x)}=e^{E(L)}=k\,p\,w^{1/\alpha(p)} .
\]
It remains to check $k\,p\,w^{1/\alpha(p)}\le w/p$, i.e.\ $kp^2\le w^{-(1/\alpha(p)-1)}$. Here $1/\alpha(p)-1=(K-2L)/L\ge0$, so the claim reads $e^{K-2L}\le e^{W(K-2L)/L}$, which holds because $W/L\ge1$ (for $L=0$ the integration range is empty).

\emph{Case 2: $L>K/2$, i.e.\ $p<k^{-1/2}$.} Split at $u=k^{-1/2}$. On the uncapped part $x\in(0,K/2]$ we have $x^2\le K^2/4\le WK/2$ because $W\ge L>K/2$; hence $E'\ge1$ again and the integral is at most $e^{E(K/2)}=k\cdot k^{-1/2}\cdot w=\sqrt k\,w\le w/p$, using $p<k^{-1/2}$. On the capped part $u\in[p,k^{-1/2}]$ the integrand is $kw$, contributing at most $k\,w\,k^{-1/2}=\sqrt k\,w\le w/p$. Adding the two parts gives $2w/p$.
\end{proof}

\begin{lemma}[work of the other arms]\label{lem:ret-work}
Let $Y=\sum_{i\ne\ast}\mathbf 1\{U_i\ge p\}\,\ell_i(U_i)$ and $C=\{Y\le H/4\}$, and let $D$ be the event that for some $i\ne\ast$ with $U_i\ge p$ the truncated test of arm $i$ ends by reaching $v$. Then $\E[Y]\le5H/64$ and $\Pr(C)\ge11/16$; moreover $C$ and $D$ are determined by $(U_i)_{i\ne\ast}$ and are therefore independent of $U_\ast$.
\end{lemma}
\begin{proof}
By Lemma~\ref{lem:ret-len} and Lemma~\ref{lem:ret-int},
\[
\E[Y]=\sum_{i\ne\ast}\int_p^1\ell_i(u)\,du\le k+H\cdot k\int_p^1w^{1/\alpha(u)}du\le k+H\frac{2w}p=k+\frac H{64},
\]
since $2w/p=2\cdot\frac q{256}\cdot\frac2q=\frac1{64}$. With $k\le H/16$ this is at most $5H/64$, and Markov's inequality gives $\Pr(Y>H/4)\le\frac{5H/64}{H/4}=\frac5{16}$. Both $Y$ and $D$ are functions of $(U_i)_{i\ne\ast}$ and of the fixed curves and noise table.
\end{proof}

\begin{lemma}[the comparator reaches $v$ quickly]\label{lem:ret-fast}
Let $n_1=\lceil H(v/(a\mu))^{1/\beta}\rceil$. Then $\hat g_\ast(n_1)\ge v$ and $n_1\le H/8$.
\end{lemma}
\begin{proof}
$g_\ast(n_1)\ge\mu(n_1/H)^\beta\ge v/a$, so $\hat g_\ast(n_1)\ge a g_\ast(n_1)\ge v$. Also $v/(a\mu)=cq/a\le2cq=w\le1$ and $1/\beta\ge1$, so $(v/(a\mu))^{1/\beta}\le w\le\frac1{256}$ and $n_1\le1+\frac H{256}\le\frac H8$ for $H\ge32$.
\end{proof}

\begin{proof}[Proof of Theorem~\ref{thm:RET}]
All arms are tested in decreasing order of their weights. We bound the expected reward on two disjoint events.

\emph{Branch 1: $C\cap D$.} Let $j$ be the arm with $U_j\ge p$, $j\ne\ast$, whose truncated test ends by reaching $v$ and whose weight is largest among such arms; $D$ says that $j$ exists. Consider the arms tested before $j$ (those with $U_i>U_j$). Every such arm $i\ne\ast$ has $U_i\ge p$ and, by maximality of $U_j$, its truncated test ends by elimination, so it consumes exactly $\ell_i(U_i)$ pulls; together they consume at most $Y\le H/4$ pulls. The comparator, if tested before $j$, consumes at most $\ell_\ast(U_\ast)\le1+Hw\le H/8$ pulls unless it reaches $v$; if it does reach $v$, the committed arm has running maximum at least $v$ and we are done. Note that $Y$ already accounts for $j$'s own truncated test, since $U_j\ge p$ and $j\ne\ast$. Hence, if the comparator has not reached $v$, the total work spent up to and including the test of $j$ is at most $Y+\ell_\ast(U_\ast)\le\frac H4+\frac H8=\frac{3H}8\le\lfloor H/2\rfloor$, the last step because $H\ge32$; the probe budget is therefore not exhausted before $j$ reaches $v$. In either case the arm committed to has running maximum at least $v$, hence true value at least $v/b$ at the time of commitment; as the true curves are nondecreasing, every one of the $r=H-\lfloor H/2\rfloor\ge H/2$ committed pulls earns at least $a\cdot v/b$. The branch therefore contributes at least $\frac ab\frac H2v$.

\emph{Branch 2: $C\cap\neg D\cap\{p\le U_\ast\le q\}$.} All arms tested before the comparator have $U_i>U_\ast\ge p$ and, by $\neg D$, are eliminated, consuming at most $Y\le H/4$ pulls in total; since $\lfloor H/2\rfloor-H/4\ge H/8$ for $H\ge32$, the comparator is tested with at least $H/8$ of the probe budget left. By Lemma~\ref{lem:ret-safe} it is never eliminated, so it is pulled until the probe budget is exhausted, and by Lemma~\ref{lem:ret-fast} it reaches observed value $v$ within $n_1\le H/8$ pulls. Arms with $U_i<U_\ast$ are never tested, and arms with $U_i\ge p$ did not reach $v$; hence the comparator is the unique arm with running maximum at least $v$ and is committed. Its $r\ge H/2$ committed pulls earn at least
\[
a\sum_{j=1}^{r}g_\ast(n_\ast+j)\ \ge\ a\sum_{j=1}^{r}g_\ast(j)\ \ge\ \frac{a\mu}{H^\beta}\cdot\frac{r^{1+\beta}}{1+\beta}\ \ge\ \frac{a\mu H}{2^{1+\beta}(1+\beta)}\ \ge\ \frac{a\mu H}8 .
\]

\emph{Combining.} Put $x=\Pr(C\cap D)$ and $y=\Pr(C\cap\neg D)$; by Lemma~\ref{lem:ret-work} the event $\{p\le U_\ast\le q\}$, of probability $q-p=q/2$, is independent of $C$ and $D$. The two branches are disjoint, so
\[
\E[\text{reward}]\ \ge\ \frac ab\frac H2v\,x+\frac{a\mu H}8\cdot\frac q2\,y
=\mu Hq\Big[\frac{ac}{2b}x+\frac a{16}y\Big]
\ \ge\ \mu Hq\cdot\frac{ac}{2b}\,(x+y),
\]
because $\frac{ac}{2b}\le\frac a{16}$. Finally $x+y=\Pr(C)\ge\frac{11}{16}\ge\frac12$, and $\frac{c}{2}\cdot\frac12=\frac1{2048}$, giving $\E[\text{reward}]\ge\frac a{2048b}\mu Hq$.

\emph{Known scale, known horizon.} Take $H=T$, $\mu=m$: the hypothesis is exactly $\LE(\beta)$ for the comparator, and $R\le mT$ gives $\E[\widehat{\ALG}_T]\ge\frac a{2048b}Rq_\beta$.

\emph{Known scale, unknown horizon.} Run $\RET(H_j,\mu)$ in phases $H_j=16k\cdot2^j$ with $\mu=m/4$, never resetting true pull counts. Let $H$ be the last completed phase; as in Theorem~\ref{thm:TGp}, $\sum_{\ell\le j+1}H_\ell=4H-16k>T$, so $H>T/4$. Conditionally on the history, the local curves $g_i(t)=f_i(a_i+t)$ are nonnegative, nondecreasing and discretely concave (the first increment dominates the second because $f_i(a_i+2)-f_i(a_i+1)\le f_i(1)\le f_i(a_i+1)$), and
\[
g_\ast(t)\ \ge\ f_\ast(t)\ \ge\ m(t/T)^\beta\ \ge\ \frac m4\Big(\frac tH\Big)^\beta
\]
because $(H/T)^\beta\ge H/T>1/4$. The theorem applied to this phase gives $\E[\text{phase}]\ge\frac a{2048b}\frac m4Hq_\beta\ge\frac a{2048b}\cdot\frac{mT}{16}q_\beta\ge\frac a{32768\,b}Rq_\beta$, all other phases contributing nonnegative rewards.
\end{proof}

\section{A blocked confidence-key algorithm for small noise}\label{app:blocked}

We now prove Proposition~\ref{prop:W48}. Let $\bar\varepsilon\in[\varepsilon,\tfrac12]$ be the noise bound supplied to the algorithm, $\bar a=1-\bar\varepsilon$, $\bar b=1+\bar\varepsilon$ and $c_{\bar\varepsilon}=\frac{\bar b}{\bar a}-\frac{\bar a}{\bar b}$.

\paragraph{The algorithm $\SRMPc(k,T,\bar\varepsilon)$.}
If $k=1$, always pull the unique arm. Otherwise set
\[
L=\max\{1,\lfloor T/(16k)\rfloor\},\qquad H=\lfloor T/L\rfloor,
\]
and run $\SRMP(k,H)$ on \emph{macro-steps}, where one macro-step consists of $L$ consecutive pulls of the chosen arm, with the following two changes: (i) the observation attached to the $n$-th macro-step of arm $i$ is the observation of the last pull of that block, written $\hat g_i(n)$; (ii) the probe key is the \emph{confidence key}
\[
\mu_i(n)=\min_{t\le n}u_i(t),\qquad u_i(t)=\frac{F_i(t)}{\bar a}-\frac{F_i(t-1)}{\bar b},\qquad F_i(t)=\max_{t'\le t}\hat g_i(t'),\ F_i(0)=0,
\]
used in the pair $(\mu_i(n_i)/Z_i,1/Z_i)$; the commitment goes to the arm with the largest $F_i(n_i)$. Any pulls left over after $H$ macro-steps are given to the committed arm. The algorithm reads $k$, $T$ and $\bar\varepsilon$ only.

\paragraph{Macro curves.} Let $\tilde g_i(n)=f_i(nL)$ for $0\le n\le H$. These are nonnegative, nondecreasing, discretely concave (an increment of $\tilde g_i$ is a sum of $L$ consecutive increments of $f_i$, and those sums are nonincreasing) and $\tilde g_i(0)=0$; the observation of the $n$-th macro-step lies in $[a\tilde g_i(n),b\tilde g_i(n)]\subseteq[\bar a\tilde g_i(n),\bar b\tilde g_i(n)]$.

\begin{lemma}[confidence key]\label{lem:conf}
For every arm and every $n$: (i) $\mu_i(n)\ge0$; (ii) $\mu_i(n)$ is nonincreasing in $n$; (iii) $\mu_i(n)\ge\tilde\delta_i(n):=\tilde g_i(n)-\tilde g_i(n-1)$; (iv) if $\tilde g_i(n)<v$ then $\mu_i(n)\le\bar\delta_i(n)+c_{\bar\varepsilon}v$, where $\bar\delta_i$ is the increment of $\min\{\tilde g_i,v\}$.
\end{lemma}
\begin{proof}
(i) $F_i$ is nondecreasing and $\bar a\le\bar b$, so $u_i(t)\ge F_i(t)(\frac1{\bar a}-\frac1{\bar b})\ge0$. (ii) is the definition of a running minimum. (iii) $F_i(t)\ge\bar a\tilde g_i(t)$ and $F_i(t-1)\le\bar b\tilde g_i(t-1)$ give $u_i(t)\ge\tilde\delta_i(t)$; the true increments are nonincreasing, so $\mu_i(n)=\min_{t\le n}u_i(t)\ge\min_{t\le n}\tilde\delta_i(t)=\tilde\delta_i(n)$. (iv) $u_i(n)\le\frac{\bar b}{\bar a}\tilde g_i(n)-\frac{\bar a}{\bar b}\tilde g_i(n-1)\le\tilde\delta_i(n)+c_{\bar\varepsilon}\tilde g_i(n)\le\bar\delta_i(n)+c_{\bar\varepsilon}v$, using $\tilde g_i(n)<v$ (whence $\bar\delta_i(n)=\tilde\delta_i(n)$), and $\mu_i(n)\le u_i(n)$.
\end{proof}

\begin{lemma}[inflated work]\label{lem:conf-work}
Let $D_H=1+c_{\bar\varepsilon}H$, $v>0$, $\lambda>0$ and $L_i^{D}=\#\{n\le H:\ (\bar\delta_i(n)+c_{\bar\varepsilon}v)/Z_i\ge\lambda\}$. Then $\E[L_i^D]\le D_Hv/\lambda$.
\end{lemma}
\begin{proof}
$\E[L_i^D]=\sum_{n\le H}\min\{1,(\bar\delta_i(n)+c_{\bar\varepsilon}v)/\lambda\}\le\frac1\lambda\big(\sum_n\bar\delta_i(n)+Hc_{\bar\varepsilon}v\big)\le\frac{v+Hc_{\bar\varepsilon}v}{\lambda}$.
\end{proof}

\begin{lemma}[blocked hitting lemma]\label{lem:conf-hit}
Let $H\ge32$, let $\ast$ be a designated arm with $\tilde g_\ast(d)\ge2v$ for some $v>0$ and integer $1\le d\le\lfloor H/4\rfloor$, and let $r=H-\lfloor H/2\rfloor$. Then
\[
\E[\widehat{\ALG}]\ \ge\ \tfrac12\min\Big\{\frac{a^2}b\,rLv,\ a\,\alpha_D\,P_\ast(rL)\Big\},\qquad
\alpha_D=\min\Big\{\theta,\ \frac H{32\,D_H\,kd}\Big\},
\]
where $\theta=\min\{1,H/(128k)\}$ is the inclusion probability of $\SRMP$ and $P_\ast(n)=\sum_{t\le n}f_\ast(t)$.
\end{lemma}
\begin{proof}
This is the proof of Lemma~\ref{lem:W44} with four substitutions, which we verify one by one.

(1) \emph{Merge structure.} By Lemma~\ref{lem:conf}(i)--(ii) the keys $\mu_i(n)$ are nonnegative and nonincreasing in the arm's own macro-count, and (the noise table being fixed in advance) they do not depend on the algorithm's randomness; hence Lemma~\ref{lem:G1} applies verbatim to the streams $(\mu_i(n)/Z_i,1/Z_i,-i,-n)$, and so does the monotone-coupling step of Lemma~\ref{lem:W44}.

(2) \emph{Counting.} Take $\lambda=32\theta kD_Hv/H$ and let $B$ be the event that some arm other than $\ast$ ends the probe prefix with running maximum at least $av$; on $\neg B$, every observed non-comparator arm has true macro-value below $v$, so Lemma~\ref{lem:conf}(iv) applies to each such observed prefix. With $C=\{\sum_{i\ne\ast}X_i\le H/32,\ \sum_{i\ne\ast}X_iL_i^{D}\le H/8\}$ we get, by Lemma~\ref{lem:conf-work}, $\E\sum_{i\ne\ast}X_iL^D_i\le\theta kD_Hv/\lambda=H/32$, so $\Pr(C)\ge1/2$ exactly as before. If the comparator has not reached true macro-value $v$, then $n_\ast<d$ and Lemma~\ref{lem:conf}(iii) with concavity gives $\mu_\ast(n)\ge\tilde\delta_\ast(n)\ge v/d$, so on $A=\{X_\ast=1,Z_\ast\le p\}$ with $p=\min\{1,H/(32\theta kD_Hd)\}$ its key is at least $\lambda$ throughout. Under this contradiction hypothesis and on $A\cap C\cap\neg B$, every extra probe of a non-comparator arm uses a key at least $\lambda$, and hence is charged to a distinct index counted by $L_i^{D}$; therefore the same count $|S|+d+\sum_{i\ne\ast}X_iL^D_i\le\frac H{16}+\frac H4+\frac H8<\lfloor H/2\rfloor$ contradicts the exhaustion of the probe budget. Hence on $A\cap C\cap\neg B$ the comparator reaches true value $v$, its running maximum is at least $av$, and it is committed. Note $\Pr(A)=\theta p=\alpha_D$.

(3) \emph{Rewards.} On $B$ the committed arm has running maximum at least $av$, hence true value at least $av/b$ at commitment, so each of its $rL$ committed pulls earns at least $a\cdot av/b$: the branch gives $\frac{a^2}brLv$. On $A\cap C\cap\neg B$ the committed arm is $\ast$ and its $rL$ committed pulls earn at least $a\sum_{j\le rL}f_\ast(j)=aP_\ast(rL)$.

(4) \emph{Combination.} Identical to Lemma~\ref{lem:W44}: the two events are disjoint, $\Pr(A\cap C\cap\neg B)\ge\alpha_D(\tfrac12-\Pr(B))_+$, and the resulting expression is at least $\tfrac12\min\{\frac{a^2}brLv,\ a\alpha_DP_\ast(rL)\}$.
\end{proof}

\begin{proof}[Proof of Proposition~\ref{prop:W48}]
The case $R=0$ is trivial. For $k=1$ the algorithm earns at least $aR$. Let $k\ge2$, set $x=HL/T$, and write $D=D(k,T,\bar\varepsilon)$.

If $L=1$, then $H=T<32k$ and $x=1$. If $L>1$, then
\[
\frac{T}{32k}\le L\le\frac{T}{16k},\qquad 16k\le H\le32k,\qquad HL>T-L\ge\Big(1-\frac1{16k}\Big)T>\frac T2 .
\]
Consequently, in both cases,
\[
x\in(1/2,1],\qquad H\le\min\{T,32k\},\qquad D_H\le D .
\]
Also $Q_\beta(k,H)=Q_\beta(k,T)$: for $L=1$ this follows from $H=T$, and for $L>1$ both quantities equal $q_\beta$.

If $H<32$, then $L>1$ is impossible because $k\ge2$. Thus $L=1$ and $T=H\le31$, with no restriction on $k$. The uniform-arm fallback earns at least
\[
\frac{aR}k\ \ge\ \frac{a\,R\,Q_\beta(k,T)}{31}\ \ge\ \frac{a^2R\,Q_\beta(k,T)}{65536\,bD},
\]
using $Q_\beta(k,T)\le T/k\le31/k$. Henceforth assume $H\ge32$.

Apply Lemma~\ref{lem:conf-hit} with $v=f_\ast(dL)/2=\tilde g_\ast(d)/2$ for an integer $1\le d\le\lfloor H/4\rfloor$ to be chosen. Using $r\ge H/2$, $rL\ge HL/2\ge T/4$, $R\le mT$ and Lemma~\ref{lem:L1}(2),
\[
\frac{a^2}brLv\ \ge\ \frac{a^2}b\cdot\frac T4\cdot\frac{f_\ast(dL)}2\ \ge\ \frac{a^2}b\cdot\frac R8\cdot\frac{f_\ast(dL)}m,
\qquad
a\alpha_DP_\ast(rL)\ \ge\ a\alpha_D\Big(\frac{rL}T\Big)^2R\ \ge\ \frac{a\alpha_DR}{16},
\]
and $\alpha_D\ge\frac1{128}\min\{1,\frac H{D kd}\}$ with $D=D(k,T,\bar\varepsilon)\ge D_H$. Since $a\ge a^2/b$ and $\frac18\ge\frac1{2048}$,
\[
\E[\widehat{\ALG}_T]\ \ge\ \frac{a^2}b\cdot\frac R{4096}\cdot\min\Big\{\frac{f_\ast(dL)}m,\ 1,\ \frac H{Dkd}\Big\}.
\]
By $\LE(\beta)$ and $dL/T=(d/H)x\ge(d/H)/2$ we have $f_\ast(dL)/m\ge(dL/T)^\beta\ge\frac12(d/H)^\beta$ (using $\beta\le1$), and $\frac H{Dkd}\ge\frac1D\cdot\frac H{kd}$. Choosing the same $d$ as in the proof of Theorem~\ref{thm:W46} (with $H$ in place of $T$) gives $\min\{(d/H)^\beta,\frac H{kd}\}\ge Q_\beta(k,H)/8$, hence
\[
\E[\widehat{\ALG}_T]\ \ge\ \frac{a^2}b\cdot\frac R{4096}\cdot\frac1{2D}\cdot\frac{Q_\beta(k,H)}8\ =\ \frac{a^2}{65536\,b\,D}\,R\,Q_\beta(k,T).
\]
For the limit statement, fix $k$ and take $\bar\varepsilon=\varepsilon$: $D(k,T,\varepsilon)\le1+\frac{4\varepsilon}{1-\varepsilon^2}\cdot32k\to1$ and $a^2/b\to1$ as $\varepsilon\downarrow0$, uniformly in $T$ and in the instance.
\end{proof}

\bibliographystyle{plain}

\end{document}